\documentclass[final]{l4dc2027}

\title[STAC-MPC: A Contraction Analysis Approach]{Stable Transformer-Actor-Critic Model Predictive Control: A Contraction Analysis Approach }
\usepackage{times}

\usepackage{amsmath,amssymb,amsfonts}
\usepackage[utf8]{inputenc}
\usepackage{bm}
\usepackage{textcomp}
\usepackage{stfloats}
\usepackage{url}
\usepackage{verbatim}
\usepackage{siunitx}
\usepackage{balance}
\usepackage{booktabs}  

\usepackage{wrapfig}

\newtheorem{assumption}{Assumption}

\newtheorem{defi}{Definition}
\newtheorem{lem}{Lemma}
\newtheorem{rema}{Remark}

\DeclareMathOperator*{\diag}{diag}
\renewcommand*{\arraystretch}{1.5}
\usepackage{lipsum}
\usepackage{mathtools}
\usepackage{subcaption}
\usepackage{graphicx}

\usepackage{tikz}
\usetikzlibrary{positioning, arrows.meta, fit, backgrounds, calc}

\usepackage{etoolbox}
\usepackage{titlesec}
\titlespacing*{\section}
  {0pt}                     
  {.2em plus .3em minus .2em} 
  {.2em plus .2em}           

\titlespacing*{\subsection}
  {0pt}                     
  {.2em plus .5em minus .1em} 
  {.2em plus .1em}           

\AtBeginDocument{%
  \setlength{\abovedisplayskip}{4pt plus 1pt minus 1pt}%
  \setlength{\belowdisplayskip}{4pt plus 1pt minus 1pt}%
  \setlength{\abovedisplayshortskip}{2pt plus 1pt minus 1pt}%
  \setlength{\belowdisplayshortskip}{2pt plus 1pt minus 1pt}%
}

\author{%
 \Name{Antonio Marino} \Email{antonio.marino@cl.cam.ac.uk}\\
 \addr University of Cambridge
 \AND
 \Name{Valerio Modugno} \Email{v.modugno@ucl.ac.uk}\\
 \addr University College London%
 \AND
 \Name{Marco Cognetti} \Email{marco.cognetti@laas.fr}\\
 \addr LAAS-CNRS, Université de Toulouse, CNRS, UPS%
}

\begin{document}

\maketitle

\begin{abstract}%
Hybrid optimization-learning frameworks, such as Actor-Critic Model Predictive Control (MPC), rely on dynamically adjusting cost parameters to manage complex control landscapes. In this paper, we introduce a certifiably robust Transformer-Actor-Critic MPC architecture that leverages history-aware sequence models to shape local optimization while maintaining strict physical safety guarantees. First, we mathematically constrain the Transformer sequence models to satisfy global incremental Input-to-State Stability ($\delta$ISS). Second, we integrate these neural Lipschitz bounds with Riemannian contraction theory, deriving a novel coupled small-gain condition that enforces uniform closed-loop contraction stability across the interconnected system. Evaluated on quadrotor control and decentralized multi-agent navigation tasks subjected to severe observation noise and process disturbances, our closed-loop stable formulation successfully translates theoretical bounds into physical safety. 
\footnote{Project page: https://seaingant.github.io/stac-mpc/}
\end{abstract}

\section{Introduction}
Synthesizing control strategies for nonlinear physical systems remains a fundamental challenge. Purely model-based approaches, such as standard Model Predictive Control (MPC), offer strong theoretical guarantees but frequently suffer from performance degradation in practice, as their success relies heavily on analytical models that inevitably fail to capture complex unmodeled dynamics. Conversely, pure learning methodologies~\citep{bertsekas2024model} can efficiently approximate dynamic programming to handle these physical complexities, yet they notoriously lack rigorous stability guarantees. This dichotomy has driven the development of hybrid optimization-learning frameworks that integrate neural components into MPC to compensate for model mismatch and improve computational tractability of long-horizon tasks \citep{hansen2024td, grandesso2023cacto} while maintaining certifiable safety properties~\citep{hewing2020learning, dzhumageldyev2025safe}. Consequently, establishing formal theoretical guarantees for these hybrid architectures is an active research area, with recent efforts exploring learned neural dynamics~\citep{bonassi2024nonlinear} and contraction theory to guarantee robust closed-loop stability without requiring restrictive terminal conditions~\citep{alamir2017contraction, wei2022discrete}.

Even with known dynamics, non-convex optimization often necessitates constraint relaxation. While learning reference trajectories shifts this complexity out of the online solver, Actor-Critic MPC~\citep{romero2024actor} generalizes this approach by dynamically adjusting the linear and quadratic coefficients of the MPC cost function. Enabled by differentiable MPC formulations~\citep{adabag2025differentiable, amos2018differentiable, amatucci2025primal}, this network shapes the local optimization landscape. However, formal guarantees for Actor-Critic MPC remain remarkably scarce, and current implementations rely on simple Multi-Layer Perceptrons (MLPs).

MLPs inherently struggle to capture the temporal consistencies required in environments with time-correlated disturbances. Sequence models like Transformers~\citep{geneva2022transformers, wang2026transformer} provide a natural framework to predict context-aware MPC weights from state trajectories. Furthermore, integrating Transformers into optimization frameworks is highly relevant for future applications, as they form the backbone of modern Vision-Language-Action (VLA) models~\citep{kawaharazuka2025vision}. While recent works have analyzed the static Lipschitz continuity of self-attention layers~\citep{kim2021lipschitz, dasoulas2021lipschitz} and ODE-based asymptotic stability of Transformers~\citep{zhong2022neural, abella2024asymptotic}, these results either address open-loop feedforward sensitivity or rely heavily on restrictive unit-norm projections. A critical gap remains in establishing closed-loop dynamical robustness notions, such as incremental Input-to-State Stability ($\delta$ISS)~\citep{d2023incremental, marino2024input} and Riemannian contraction~\citep{tedrake2010lqr, tsukamoto2021contraction}, for sequence models operating in feedback control loops.

To address this gap, this paper makes two intertwined contributions. First, we prove that Transformer architectures can be mathematically constrained to satisfy global $\delta$ISS with respect to the induced infinity norm. Second, we propose a novel hybrid Transformer-Actor-Critic MPC architecture. We demonstrate that the derived $\delta$ISS Lipschitz bounds for the Transformer directly enable a coupled small-gain condition for the physical plant. By enforcing this closed-loop contraction stability during training, our framework yields a control policy that is both highly performant and certifiably robust. 

\section{Preliminaries and Problem Formulation}
\label{sec:prelim}
Consider the following discrete-time nonlinear system:
\begin{equation}
    x_{k+1} = \Phi(x_k, u_k) + d_k(x_k)
    \label{eq:system}
\end{equation}
where $x_k \in \mathcal{X} \subseteq \mathbb{R}^{n_x}$ is the system state, $u_k \in \mathcal{U} \subseteq \mathbb{R}^{n_u}$ is the control input, and $d_k : \mathcal{X} \rightarrow \mathbb{R}^{n_x}$ is an additive disturbance acting on the system while the state measurements $\tilde{x}_k$ are subjected to the noise $n_k : \mathcal{X} \rightarrow \mathbb{R}^{n_x}$ such that $\tilde{x}_k = x_k + n_k$. The disturbance is assumed to be uniformly bounded, i.e., there exists $\bar{d} \in \mathbb{R}_{> 0}$ such that $\sup_{x \in \mathcal{X},\, k \in \mathbb{Z}_{\geq 0}} \|d_k(x)\|_{\infty} \leq \bar{d}$.
\begin{assumption}
    The system state $x \in \mathcal{X}$ is uniformly bounded by a constant limit constraint satisfying $\|x\|_{\infty} \leq \tilde{X}$, and analogously, the action space mandates that the applied control is subject to $\|u\|_{\infty} \leq \tilde{U}$.    \label{assumption2}
\end{assumption}

Let $[x_{0|k}, \dots, x_{T|k}]$ denote a predicted state sequence whose element $(\cdot)_{i | k}$ denotes the $i$-step-ahead prediction from time step $k$, initialized by the measured state $x_{0|k} \triangleq x_k$. Given a desired state trajectory $x^d_{k} \in \mathcal{X}$ and the state-action pair $\omega_{i|k} = [ x_{i|k}^{\top},u_{i|k}^{\top}]^{\top}$, we formulate the MPC problem over horizon $T$ as:
\begin{equation}
    \begin{aligned}
        \min_{\omega} \quad J_{Y} &= \sum_{i=0}^{T} \left(\omega_{i|k}^{\top}Q_{i|k}(x_k, x^d_{k})\omega_{i|k} + P_{i|k}^{\top}(x_k, x^d_{k})\omega_{i|k}\right) \\
    \text{subject to} \quad & x_{i+1|k} = \Phi(x_{i|k}, u_{i|k}), \quad x_{0|k} = x_k, \\
        & h(\omega_{i|k}) \leq 0, \quad u_{i|k} \in \mathcal{U}, \quad x_{i|k} \in \mathcal{X},
    \end{aligned}  
    \label{eq:problem}
\end{equation}
where $\omega = [\omega^{\top}_{0\vert{}k}, \dots, \omega^{\top}_{T\vert{}k}]^{\top}$ concatenates the decision variables subject to stage-wise constraints $h(\omega) \leq 0$. The local optimization landscape is governed by the sequence of cost parameters $Y = ( (Q_{i\vert{}k}, P_{i\vert{}k}))_{i=0}^T$. The matrix $Q_{i\vert{}k} = \text{blockdiag}(Q_{x, i\vert{}k}, Q_{u, i\vert{}k}) \in \mathbb{R}^{(n_x+n_u) \times (n_x+n_u)}$ comprises strictly positive definite state $Q_{x,i|k} \succeq q_{\min} I$ and input $Q_{u,i|k} \succeq q_{\min} I$ cost matrices ($q_{\min} > 0$), while $P_{i\vert{}k} \in \mathbb{R}^{n_x+n_u}$ is the linear cost vector. The cost parameter $Y$ is learned within an actor-critic framework~\citep{romero2024actor}, where the actor predicts parameters driving the system toward desired trajectories through the linear shift $P_{i \mid k}$ (which implicitly defines unconstrained local targets $-\frac{1}{2}Q_{i|k}^{-1}P_{i|k}$) at the rate shaped by $Q_{i \mid k}$.
To bypass active-set discontinuities, constraints $h(\omega) \leq 0$ are relaxed via a smooth Softplus penalty $So(\cdot)$ with weight $\gamma > 0$, yielding the relaxed objective $J_{Y} = \sum_{i=0}^{T} [ \omega_{i|k}^{\top}Q_{i|k}\omega_{i|k} + P_{i|k}^{\top}\omega_{i|k} ] + \gamma \mathbf{1}^{\top} So(h(\omega))$. 

\begin{wrapfigure}{r}{0.6\textwidth}
\centering
\resizebox{\linewidth}{!}{
\begin{tikzpicture}[
    node distance = 0.8cm and 0.8cm, 
    base/.style = {rectangle, draw=black!80, text centered, rounded corners, minimum height=0.7cm, font=\sffamily\footnotesize, thick},
    nn/.style = {base, fill=blue!10, text width=1.8cm},
    transformer/.style = {base, fill=orange!10, text width=2.2cm, minimum height=0.8cm, font=\sffamily\bfseries\footnotesize},
    mpc/.style = {base, fill=gray!10, text width=2.2cm, minimum height=0.7cm},
    system/.style = {rectangle, inner sep=0pt}, 
    data/.style = {rectangle, text centered, text width=1.4cm, font=\sffamily\scriptsize},
    arrow/.style = {-{Latex[length=2mm, width=2mm]}, thick},
    line/.style = {thick},
    dot/.style = {circle, fill, inner sep=1pt}
]

    
    \node [data] (ref) {Ref $x^d_{k-L:k}$};
    
    \node [data, below=1.5cm of ref] (noise) {Noise $n_{k-L:k}$};
    
    \node [circle, draw, inner sep=1pt, right=1.1cm of noise] (sum_noise) {$+$};
    
    \path (ref) -- (sum_noise) coordinate[midway] (mid_in);
    
    \node [nn, right=0.3cm of ref] (mlp_in) {Initial MLP \\ \tiny(Feature Ext.)};
    
    \node [transformer, right=1.0cm of mlp_in] (transformer) {Transformer \\ Sequence Model};
    
    \node [nn, right=1.0cm of transformer] (mlp_out) {Final MLP \\ \tiny(Projection)};
    
    \node [data, right=0.6cm of mlp_out, text width=1.5cm] (costs) {Reshaped to\\ $\hat{Q}, \hat{P}$ matrices};

    
    \node [mpc, below=1.6cm of transformer, xshift=-0.5cm] (mpc) {MPC};
    
    \node [circle, draw, inner sep=1pt, right=0.6cm of mpc] (sum_dist) {$+$};
    
    \node [data, above=0.5cm of sum_dist] (dist) {$d_k$};
    
    \node [system, right=0.6cm of sum_dist] (system) {\includegraphics[width=2.5cm]{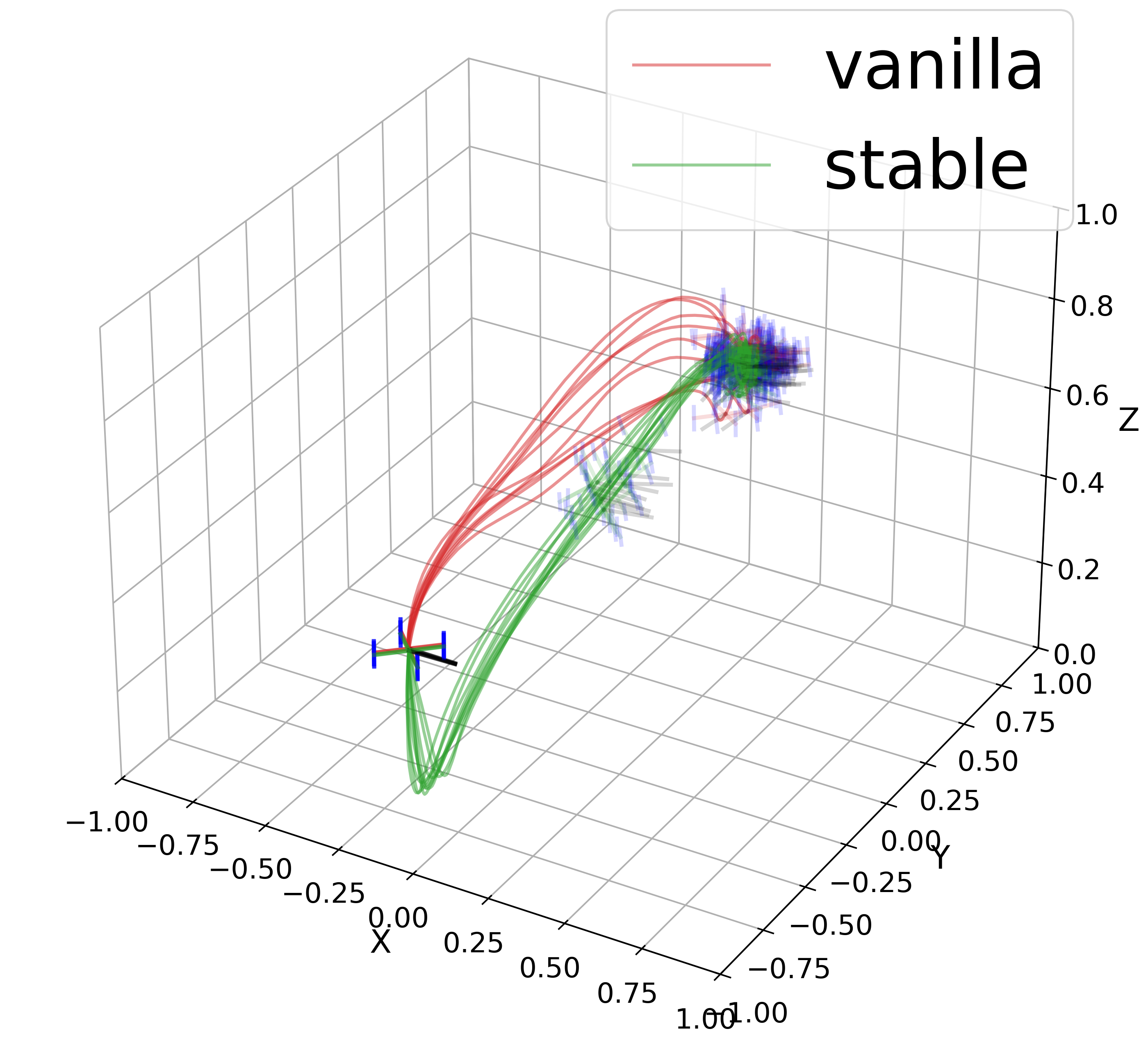}};

    
    \draw [arrow] (ref.east) -- ($(mlp_in.west)$);
    
    \draw [arrow] (noise) -- (sum_noise) node[midway, right, font=\sffamily\tiny] {$+$};
    
    \draw [arrow] (sum_noise.north) -- ($(mlp_in.south)$) node[near start, left, font=\sffamily\scriptsize] {$\tilde{x}_{k-L:k}$};
        
    \draw [arrow] (mlp_in) -- (transformer) node[midway, above, font=\sffamily\scriptsize] {$p_k$};
    \draw [arrow] (transformer) -- (mlp_out) node[midway, above, font=\sffamily\scriptsize] {$z_k$};
    \draw [arrow] (mlp_out) -- (costs);
    
    \draw [arrow] (costs.south) -- +(0,-0.5cm) -| (mpc.north);
    
    \draw [arrow] (mpc) -- (sum_dist) node[midway, above, font=\sffamily\scriptsize] {$u_k$};
    
    \draw [arrow] (dist) -- (sum_dist) node[midway, right, font=\sffamily\tiny] {$+$};
    
    \draw [arrow] (sum_dist) -- (system);

    
    \coordinate (sys_out) at ($(system.east) + (0.8cm, 0)$);
    \draw [line] (system.east) -- (sys_out) node[midway, above, font=\sffamily\scriptsize] {$x_{k+1}$};
    
    \coordinate (fb_down_right) at ($(sys_out) + (0, -1.5cm)$);
    \coordinate (fb_down_left)  at (sum_noise |- fb_down_right);
    
    \draw [line] (sys_out) -- (fb_down_right);
    \draw [line] (fb_down_right) -- (fb_down_left);
    
    \draw [arrow] (fb_down_left) -- (sum_noise) node[near end, right, font=\sffamily\tiny] {$+$};
    
    \coordinate (mpc_tap) at (mpc.south |- fb_down_right);
    \node [dot] at (mpc_tap) {}; 
    \draw [arrow] (mpc_tap) -- (mpc.south) node[midway, right=1mm, font=\sffamily\scriptsize] {$\tilde{x}_k$};

    \begin{scope}[on background layer]
        
        \node [
            rectangle, 
            draw=gray!50, 
            thick, 
            dashed,
            rounded corners, 
            inner xsep=0.2cm, 
            inner ysep=0.4cm, 
            fit=(mlp_in) (transformer) (mlp_out) (costs),
            fill=gray!5
        ] (ai_box) {};
        \node [anchor=south west, text=gray, font=\sffamily\bfseries\tiny, yshift=1mm] at (ai_box.north west) {Predictive Network ($\theta$)};
        
        \node [
            rectangle, 
            draw=teal!60, 
            thick, 
            dashed,
            rounded corners, 
            inner xsep=1.0cm, 
            inner ysep=0.2cm, 
            fit=(mpc) (dist) (system),
            fill=teal!5
        ] (ctrl_box) {};
        \node [anchor=south west, text=teal!80, font=\sffamily\bfseries\tiny, yshift=1mm] at (ctrl_box.north west) {Controlled System};
        
    \end{scope}

\end{tikzpicture}
}
\caption{Closed-loop architecture mapping the predicted parameter matrices to the MPC.}
\label{fig:closed_loop_arch}
\end{wrapfigure}

Departing from standard actor-critic formulations, our objective is to synthesize a tracking policy that dynamically shapes $Y$ to reject physical disturbances ($d_k$) and time-correlated sensor noise ($n_k$). To achieve this, we propose two core changes. \textbf{First}, we introduce the history-aware architecture in Fig.~\ref{fig:closed_loop_arch}: an initial MLP embeds an $L$-step history $\{x^d, \tilde{x}\}_{k-L:k}$ into $p_k \in \mathbb{R}^{L \times D}$, which a Transformer block processes to update a temporal representation $z_k \in \mathbb{R}^{L \times D}$ that compensates for partial observability. A terminal MLP projects $z_k$ into $Y$, enforcing $Q \succeq q_{\min}I$ via $\diag(Q) = q_{\min}\mathbf{1} + So(Q_{d}(\tilde{x}_k,x_k^d))$. \textbf{Second}, to formally guarantee bounded closed-loop behavior under the resulting dynamically parameterized MPC law $u_k(\tilde{x}_k, Y)$, we leverage Riemannian contraction theory. On this purpose, we introduce the following lemma:


\begin{lem}[\cite{tsukamoto2021contraction}]
\label{lemma:contraction}
Let $x^1_k$ and $x^2_k$ be two state trajectories of system \eqref{eq:system}. If there exists a uniformly positive definite contraction metric $M_k(x_k) \succeq \underline{m}I \succ 0$ such that for a contraction rate $\alpha \in (0,1)$ and $\forall x_k \in \mathcal{X}$:
\begin{equation}
    \frac{\partial \Phi_k}{\partial x_k}^{\top} M_{k+1}(x_{k+1}) \frac{\partial \Phi_k}{\partial x_k} \preceq \alpha^2 M_k(x_k),
\end{equation}
then the Riemannian geodesic distance $d_\mathcal{M}(x_k^1, x_k^2) \triangleq \inf_{\ell} \int_{0}^{1} \left\| \frac{\partial \ell(s, k)}{\partial s} \right\|_{M_k(\ell(s, k))} ds$ connecting $\ell(0,k)=x_k^1$ and $\ell(1,k)=x_k^2$ satisfies:
\begin{equation}
    d_\mathcal{M}(x_{k}^1, x_{k}^2) \leq \alpha^k d_\mathcal{M}(x_{0}^1, x_{0}^2) + \bar{d}\sqrt{n_x \bar{m}}\frac{1-\alpha^k}{1-\alpha},
\end{equation}
with $\bar{m} = \sup_{\mathcal{X},k} \lambda_{\max}(M_{k})$.
\end{lem}
An implication of Lemma~\ref{lemma:contraction} is that, for $M_k \succeq \underline{M} \triangleq \underline{m}I$, trajectories converge to the robust invariant set $\Omega := \{ x^1,x^2 \in \mathcal{X} : \| x^1- x^2 \|_{\underline{M}} \leq \bar{d}\sqrt{n_x \bar{m}}/(1-\alpha) \}$~\citep{singh2017robust}. 
Lemma~\ref{lemma:contraction} defines the contraction property of a closed-loop system under control action, dynamically parametrized in our framework (Fig.~\ref{fig:closed_loop_arch}): a Transformer continuously updates the MPC cost landscape $Y_k$ based on a shifting window of historical observations. This creates a highly coupled feedback loop between a physical optimizer and a neural sequence model. Therefore, our main contribution is to apply contraction analysis by temporarily decoupling these domains. First, in Section~\ref{sec:C-MPC}, we isolate the physical subsystem to establish its nominal Riemannian contraction rate and its exact KKT sensitivity to dynamic parameter variations. Next, in Section~\ref{sec:iss-transf}, we isolate the neural subsystem, proving that the sequence-processing Transformer acts as a strictly bounded dynamical system satisfying incremental Input-to-State Stability ($\delta$ISS). Finally, in Section~\ref{sec:small-gain}, we reunite the subsystems, deriving a coupled small-gain condition that mathematically guarantees their continuous interaction strictly dampens exogenous perturbations, yielding uniform closed-loop contraction for the complete framework.



\section{Contraction of the Actor-Critic MPC}
\label{sec:C-MPC}
To guarantee closed-loop stability, we first characterize how state variations $dx_k$ and Transformer parameter variations $dY_k$ propagate through the MPC policy. Defining $A_{0|k} \triangleq \nabla_x \Phi(x^*_{0|k}, u^*_{0|k})$ and $B_{0|k} \triangleq \nabla_u \Phi(x^*_{0|k}, u^*_{0|k})$ as the plant Jacobians at the initial stage, the exact differential of the nominal dynamics \eqref{eq:system} under the optimal policy $u^*_{0|k}(x_k, Y_k)$ is:
\begin{equation}
    dx_{k+1} = (A_{0|k} + B_{0|k} \nabla_{x_k} u^*_{0|k}) dx_k + B_{0|k} \nabla_Y u^*_{0|k} dY_k.
    \label{eq:exact_diff}
\end{equation}
Let $A_{\text{cl},k} \triangleq A_{0|k} + B_{0|k} \nabla_{x_k} u^*_{0|k} = A_{0|k} - B_{0|k} \widetilde{K}_{0|k}$ denote the true closed-loop MPC Jacobian, where $\widetilde{K}_{0|k} \triangleq -\nabla_{x_k} u^*_{0|k}$ is the exact KKT state-feedback gain at fixed $Y_k$.

We perform our sensitivity and contraction analysis in the interior of the admissible operating region where boundary penalties are inactive (any residual smooth curvature $\gamma \mathbf{1}^{\top} So(h(\omega_{i|k}))$ is absorbed into the stage Hessian blocks below and omitted from the Lagrangian for clarity). Adopting the dual multiplier sequence $\beta = [\beta_{0|k}^{\top}, \dots, \beta_{T|k}^{\top}]^{\top}$ for the stacked initial state and dynamics equality constraints $G(\omega; x_k) \triangleq [x_k^{\top} - x_{0|k}^{\top},\, (\Phi(x_{0|k}, u_{0|k}) - x_{1|k})^{\top}, \dots]^{\top} = 0$, the interior Lagrangian evaluates to: \vspace{-0.4em}
\begin{equation}
    \mathcal{L}(\omega, \beta; x_k) = \sum_{i=0}^{T} \left(\omega_{i|k}^{\top}Q_{i|k}\omega_{i|k} + P_{i|k}^{\top}\omega_{i|k}\right) + \beta^{\top} G(\omega; x_k).
\end{equation}\vspace{-1.2em}
\begin{assumption}[Strong Convexity, Regularity, and Stabilizability]
\label{ass:kkt_invertibility}
The system dynamics $\Phi(x,u)$ are three times continuously differentiable on the compact domain $\mathcal{X} \times \mathcal{U}$, and the equality constraints $G(\omega; x_k) = 0$ satisfy the Linear Independence Constraint Qualification (LICQ) at any optimal solution. Because $x_{i+1|k}$ enters $G$ linearly, the primal Lagrangian Hessian at an optimal primal-dual pair $(\omega^*, \beta^*)$ is strictly block-diagonal across stages:
\begin{equation}
    \mathbb{H}(\omega^*, \beta^*) \triangleq \nabla^2_{\omega\omega} \mathcal{L}(\omega^*, \beta^*; x_k) = 2\,\text{blockdiag} \left(\mathcal{H}_{0|k}, \dots, \mathcal{H}_{T|k}\right) \succeq 2\mu_2 I \succ 0,
\end{equation}
with uniform curvature floor $\mu_2 > 0$, terminal block $\mathcal{H}_{T|k} = Q_{T|k}$, and stage-wise blocks for $i = 0, \dots, T-1$: \vspace{-1em}
\begin{equation}
    \mathcal{H}_{i|k} \triangleq Q_{i|k} + \frac{1}{2}\sum_{r=1}^{n_x} \beta^*_{i+1|k,r} \nabla^2_{\omega_{i|k}\omega_{i|k}} \Phi_r(x^*_{i|k}, u^*_{i|k}) = \begin{bmatrix} \mathcal{H}_{xx,i|k} & \mathcal{H}_{xu,i|k} \\ \mathcal{H}_{ux,i|k} & \mathcal{H}_{uu,i|k} \end{bmatrix} \succeq \mu_2 I.
\end{equation}
Furthermore, the linearized pairs $(A_{i|k}, B_{i|k})$ along the predicted trajectory and $(A_T, B_T)$ at the terminal target are uniformly stabilizable.
\end{assumption}\vspace{-0.8em}
This curvature condition is readily satisfied by selecting the cost floor $q_{\min} > 0$ sufficiently large to dominate any local negative curvature of the nonlinear dynamics on $\mathcal{X} \times \mathcal{U}$. Applying the Implicit Function Theorem (IFT) to the KKT residual $F(\eta; x_k, Y) \triangleq \nabla_\eta \mathcal{L}(\omega, \beta; x_k) = 0$ with $\eta = [\omega^{\top}, \beta^{\top}]^{\top}$, Assumption~\ref{ass:kkt_invertibility} ensures the KKT matrix $K = \nabla_\eta F$ remains globally non-singular with uniformly bounded inverse on $\mathcal{X}$. Isolating the parameter differential via the IFT yields:
\begin{equation}
    \frac{\partial \eta^*}{\partial Y} = -K^{-1} \nabla_Y F = -K^{-1} \begin{bmatrix} 2\diag(\omega^*) & I \\ 0 & 0 \end{bmatrix},
    \label{eq:kkt_sensitivity}
\end{equation}
from which the control sensitivity $\nabla_Y u^*_{0|k}$ in \eqref{eq:exact_diff} is directly extracted.

Next, to evaluate nominal physical transition stability, $dY_k = 0$, without conservative infinity-norm bounds, we equip the state space with a time-varying local Riemannian metric along the predicted trajectory~\citep{tsukamoto2021contraction}. We extract the terminal target $\hat{\omega}_{T|k} = [(\hat{x}_{T|k})^{\top}, (\hat{u}_{T|k})^{\top}]^{\top}$ intended by the Transformer by minimizing the unconstrained terminal stage quadratic form: $\hat{\omega}_{T|k} = -\frac{1}{2}Q_{T|k}^{-1}P_{T|k}$. Evaluating the plant Jacobians $(A_T, B_T)$ at $\hat{\omega}_{T|k}$, uniform stabilizability (Assumption~\ref{ass:kkt_invertibility}) and $Q_{T|k} \succeq q_{\min} I$ guarantee a unique positive definite terminal metric $M_{T|k} \succ 0$ solving the steady-state Discrete Algebraic Riccati Equation (DARE):
\begin{equation}
    M_{T|k} = A_T^\top M_{T|k} A_T - A_T^\top M_{T|k} B_T (Q_{u, T|k}+ B_T^\top M_{T|k} B_T)^{-1} B_T^\top M_{T|k} A_T + Q_{x, T|k}.
\end{equation}
Anchored at $M_{T|k}$, we propagate the contraction metric backward along the optimal trajectory for $i = T-1$ down to $0$ via the stage-wise Lagrangian-Hessian Difference Riccati Equation (DRE)~\citep{amos2018differentiable}:
\begin{equation}
    \begin{aligned}
        K_{i|k} &= \left(\mathcal{H}_{uu, i|k} + B_{i|k}^\top M_{i+1|k} B_{i|k}\right)^{-1} \left(\mathcal{H}_{ux, i|k} + B_{i|k}^\top M_{i+1|k} A_{i|k}\right), \\
        M_{i|k} &= \mathcal{H}_{xx, i|k} + A_{i|k}^\top M_{i+1|k} A_{i|k} - \left(\mathcal{H}_{xu, i|k} + A_{i|k}^\top M_{i+1|k} B_{i|k}\right) K_{i|k}.
    \end{aligned}
    \label{eq:DRE}
\end{equation}
The root matrix $M_{0|k}$ defines the Riemannian metric at step $k$. By Assumption~\ref{ass:kkt_invertibility} and compactness of $\mathcal{X} \times \mathcal{U}$, $M_{0|k}$ satisfies $\mu_2 I \preceq M_{0|k} \preceq \bar{m} I$ uniformly for finite $\bar{m} > 0$. We state two lemma below on the MPC contraction which proofs are provided in the Appendix.
\begin{lem}[Uniformly Bounded Metric Shift]
\label{lem:metric_shift}
Under Assumption~\ref{ass:kkt_invertibility} and the receding horizon control scheme, the temporal shift of the contraction metric, $\Delta M_{shift} \triangleq M_{0|k+1} - M_{1|k}$, admits a strict uniform bound $\sup_{x \in \mathcal{X},\, k \in \mathbb{Z}_{\geq 0}} \|\Delta M_{shift}\|_2 \leq \Delta\bar{M}$, which is proportional to the disturbance bound $\bar{d}$ along steady-state target equilibria.
\end{lem}
\begin{lem}[MPC Contraction Rate] \label{lem:contraction_rate}
Under Assumptions~\ref{assumption2}--\ref{ass:kkt_invertibility}, Lemma~\ref{lemma:contraction}, and Lemma~\ref{lem:metric_shift}, for a sufficiently large prediction horizon $T$, there exists a terminal Riccati boundary mismatch $\varepsilon_T \in [0, \mu_2)$~\citep{alamir2017contraction} such that the nominal closed-loop MPC system ($dY_k = 0$) is practically exponentially contractive on the time-varying $M$-weighted Riemannian manifold with uniform contraction rate bounded by:
\begin{equation}
    \rho_{\text{mpc}} \leq \sqrt{1 - \frac{\mu_2 - \varepsilon_T}{\bar{m}}} < 1.
\end{equation}
\end{lem}
\begin{rema}
    Although the Riemannian metric at step $k$ is evaluated locally along the optimal trajectory shaped by the Transformer, the uniform bounds $\mu_2$, $\varepsilon_T$, and $\bar{m}$ in Lemma~\ref{lem:contraction_rate}, combined with Lemma~\ref{lem:metric_shift}, establish uniform practical contraction of the nominal plant across $\mathcal{X}$ according to Lemma~\ref{lemma:contraction}. 
\end{rema}

\section{Transformer Incremental Stability}
\label{sec:iss-transf}

While Lemma~\ref{lem:contraction_rate} guarantees nominal MPC contraction, physical stability requires bounded variations in the Transformer-generated cost parameters $Y_k$. To prevent neural noise amplification from destabilizing the plant, we must first establish that the Transformer operates as a strictly bounded dynamical system that attenuates exogenous noise. To analyze the Transformer layer, we rely on its discrete dynamical system interpretation~\citep{abella2024asymptotic}, equipping sequence matrices $z \in \mathbb{R}^{L \times D}$ with the induced matrix infinity norm $\|z\|_\infty = \max_{1 \leq l \leq L} \sum_{d=1}^D |z_{l,d}|$. For the temporal state sequence $z_k \in \mathbb{R}^{L \times D}$ and observation input sequence $p_k\in \mathbb{R}^{L \times D}$, the Transformer block model follows:
\begin{equation}
\label{eq:transformer-system}
    z_{k+1} = \alpha_2 g(\alpha_1 z_k+A(z_k,p_k)) + f(g(\alpha_1 z_k+A(z_k,p_k)))
\end{equation}
with $\alpha_1,\alpha_2 \in (0, 1)$, $f$ denoting an $M$-layer MLP network with $1$-Lipschitz activation functions (e.g., $\tanh$, ReLU) and weight matrices $W_i$, $g(\cdot)$ representing token-wise Layer Normalization $g(x) = c \odot \frac{x - \mu}{\sqrt{\sigma^2 + \epsilon}} + \beta$ (where $c, \beta \in \mathbb{R}^D$ are learnable scale and shift parameters, $\sigma$ and $\mu$ are the token standard deviation and mean, $\sigma_{\min} \geq 0$ is a lower bound on $\sigma$, and $\epsilon>0$ is a stabilization constant), and $A(\cdot,\cdot)$ the multi-head attention mechanism between $z_k$ and $p_k$. Drawing upon standard multi-head attention with $H$ heads, we formulate $A(z_k,p_k)$ as follows:
\begin{equation}
    A(z_k,p_k) = \text{Concat}(\mathcal{H}_1, \dots, \mathcal{H}_H) W_O,
\end{equation}
\begin{equation}
    \mathcal{H}_h(z_k, p_k) = \text{softmax}\Bigg(\frac{(z_k W_Q^{(h)})(p_kW_K^{(h)})^\top}{\sqrt{D_h}}\Bigg)(p_kW_V^{(h)}),
\end{equation}
where $D_h = D/H$ represents the hidden dimension of each head. The matrices $W_Q^{(h)},W_K^{(h)},W_V^{(h)} \in \mathbb{R}^{D \times D_h}$ denote the linear projection weights for the queries, keys, and values for head $h$, and $W_O \in \mathbb{R}^{H \cdot D_h \times D}$ is the output projection matrix. Note that this architecture, used in reinforcement learning~\citep{parisotto2020stabilizing}, omits a final LayerNorm in favor of weighted residual connections ($\alpha_1, \alpha_2 < 1$). Unlike Recurrent Neural Networks, the Transformer does not maintain a real recurrent hidden state. However, by processing the observation history as a shifting sequence, the interplay of the attention block and residual connections over the sequence allows us to interpret the architecture as a virtual discrete-time dynamical system. Thus, the variable $z_k$ in~\eqref{eq:transformer-system} explicitly denotes a virtual state sequence updated at each step $k$ based on the new observation sequence $p_k$.

To establish the incremental stability of this network, we leverage the following definition: \vspace{-0.4em}
\begin{defi}[$\delta$ISS ~\cite{bayer2013discrete}]
\label{dISS_def}
A dynamical system is incrementally input-to-state stable with respect to the induced infinity norm if there exist functions $\beta_{\delta} \in \mathcal{KL}$ and $\gamma_{\delta} \in \mathcal{K}_\infty$ such that, for any $k \in \mathbb{Z}_{\geq 0}$, any initial states $x_0^1,x_0^2 \in \mathcal{X}$ defining the initial state deviation $\delta x_0 \triangleq x_0^1 - x_0^2$, and any input sequences $u^1,u^2 \in \mathcal{U}$ defining the input deviation $\delta u \triangleq u^1 - u^2$, the state deviation $\delta x_k \triangleq x_k^1 - x_k^2$ satisfies:
\begin{equation}
\label{eq:incremental-iss}
    \|\delta x_k\|_{\infty} \leq \beta_{\delta}(\|\delta x_0\|_{\infty},k) + \gamma_{\delta}(\|\delta u\|_{\infty}).
\end{equation}
\end{defi} \vspace{-0.4em}
Intuitively, this implies that the deviation between any two state trajectories is strictly bounded by a transient term that decays over time based on their initial separation ($\beta_{\delta}$), plus an asymptotic margin governed by the maximum difference between their applied inputs ($\gamma_{\delta}$).\vspace{-0.4em}
\begin{assumption}
The input sequence $p$ is uniformly bounded feature-wise within $\mathcal{P} \subseteq [-1/D, 1/D]^{L \times D}$, such that its induced infinity norm satisfies $\|p\|_{\infty} \leq 1$.
\label{assumption1}
\end{assumption}

\begin{theorem}
\label{dISS_stab}
Under Assumption~\ref{assumption1}, a sufficient condition for the system~\eqref{eq:transformer-system} to be $\delta$ISS is $\mathcal{A}_{\delta} < 1$, where
\begin{flalign*}
\begin{aligned}
    \mathcal{A}_{\delta} & \triangleq \left(\alpha_2+\prod_{i = 1}^{M}\| W_i \|_{\infty}\right) \bigg( \frac{2(1 - 1/D)\|c\|_{\infty}}{\sqrt{\sigma_{\min}^2+\epsilon}} \big(\alpha_1 + \mathcal{L}_{MHA, z} \big) \bigg),
\end{aligned}
\end{flalign*}
with 
\begin{equation*}
    \mathcal{L}_{MHA, z} \triangleq \|W_O\|_{\infty} \sum_{h=1}^H \left( \frac{L}{2\sqrt{D_h}} \|W_V^{(h)}\|_{\infty} \|W_Q^{(h)}\|_{\infty} \|W_K^{(h)}\|_{\infty} \right).
\end{equation*}
\end{theorem} The formal proof is reported in the Appendix. Because $g(\cdot)$ and the bounded-activation MLP $f(\cdot)$ have uniformly bounded outputs, $z_k$ evolves within a compact forward-invariant set normalized to $\|z_k\|_{\infty} \leq 1$ under appropriate weight scaling; to prevent small token standard deviations $\sigma_{\min}$ from yielding conservative bounds during training, we set $\epsilon \simeq 0.1$ in the stability regularizer. Unlike flattened MLPs, which scale poorly with horizon $L$ and destroy sequential structure, processing $p_k \in \mathbb{R}^{L \times D}$ through the sequence-to-sequence Transformer dynamically filters time-correlated noise while adhering to the $\delta$ISS bounds in Theorem~\ref{dISS_stab} required for closed-loop contraction.


\section{Interconnected Closed-Loop Stability}
\label{sec:small-gain}
We now address their continuous feedback interconnection. We follow by interconnecting the differential \eqref{eq:exact_diff} to the transformer generating $Y$. We first bound the coupled sensitivity mapping linking the neural parameter space to the physical Riemannian manifold:
\begin{equation}
\mathcal{S}_{Y} \triangleq \sup_{x_k \in \mathcal{X},\, k \in \mathbb{Z}_{\geq 0}} \|B_{0|k} \nabla_Y u_{0|k}^*\|_{\infty \to M_{1|k}} = \sup_{x_k, k} \sup_{dY \neq 0} \frac{\|B_{0|k} \nabla_Y u_{0|k}^* dY\|_{M_{1|k}}}{\|dY\|_\infty}.
\end{equation}
Since $\|\cdot\|_{M_{1|k}} \leq \sqrt{\bar{m}} \|\cdot\|_2$, submultiplicativity yields~$\|B_{0|k} \nabla_Y u_{0|k}^* dY\|_{M_{1|k}} \leq \sqrt{\bar{m}} \|B_{0|k}\|_2 \allowbreak \|\nabla_Y u_{0|k}^* dY\|_2$. Under receding-horizon control, only the first action $u_{0|k}^* = \Pi_0 \omega^*$ is applied, where the selection matrix satisfies $\|\Pi_0\|_2 = 1$. Factoring out $\|dY\|_\infty$ over the compact domain $\mathcal{X} \times \mathcal{U}$ (with $\bar{b} \triangleq \sup_{\mathcal{X}\times\mathcal{U}}\|B\|_2$) yields:
\begin{equation}
\mathcal{S}_{Y} \leq \sqrt{\bar{m}}\,\bar{b}\,\|\Pi_0\|_2 \sup_{x_k, k}\|\nabla_Y \omega^*\|_{\infty \to 2}.
\end{equation}

The solver sensitivity $\|\nabla_Y \omega^*\|_{\infty \to 2}$ is governed by the top-left primal block of the inverse KKT matrix, denoted $(K^{-1})_{11}$, acting on $J \triangleq [2\diag(\omega^*) \;\; I]$. By Assumption~\ref{ass:kkt_invertibility}, the saddle-point Schur complement satisfies $\|(K^{-1})_{11}\|_2 \leq \|\mathbb{H}^{-1}\|_2 \leq (2\mu_2)^{-1}$, so the mixed sensitivity is uniformly bounded by:
\begin{equation}
\mathcal{S}_{Y} \leq \frac{\sqrt{\bar{m}}}{2\mu_2}\,\bar{b} \sup_{\omega^*}\| J \|_{\infty \to 2}.
\end{equation}

Integrating the exact differential $dx_{k+1}$ \eqref{eq:exact_diff} along the $M_{0|k}$-minimizing geodesic, shifting $M_{1|k}$ to $M_{0|k+1}$ via Lemma~\ref{lem:contraction_rate} ($+\bar{d}_{shift}$), and incorporating the plant disturbance deviation $\delta d_k \triangleq d_k^1(x_k^1) - d_k^2(x_k^2)$ via Lemma~\ref{lemma:contraction} gives:
\begin{equation}
d_{\mathcal{M}_{0|k+1}}(x_{k+1}^1, x_{k+1}^2) \leq \rho_{mpc} d_{\mathcal{M}_{0|k}}(x_k^1, x_k^2) + \mathcal{S}_Y \|\delta Y_k\|_\infty + \sqrt{n_x \bar{m}}\|\delta d_k\|_\infty + \bar{d}_{shift}.
\end{equation}

Denote the Lipschitz constants of the initial and final MLPs ($\Gamma$ and $\Psi$) in $\|\cdot\|_\infty$ as $L_{\Gamma}$ and $L_{\Psi}$, so $\|\delta Y_k\|_\infty \leq L_\Psi \|\delta z_{k+1}\|_\infty$ and $\|\delta p_k\|_\infty \leq L_\Gamma (\max_{0 \leq \tau \leq L}\|\delta x_{k-\tau}\|_\infty + \|\delta n\|_\infty)$, where $\|\delta n\|_\infty \triangleq \sup_k \|n_k^1 - n_k^2\|_\infty$. Using $M_{0|j} \succeq \mu_2 I$ ($\|\delta x_j\|_\infty \leq \mu_2^{-1/2} d_{\mathcal{M}_{0|j}}(x_j^1, x_j^2)$), the coupled subsystem bounds read:
\begin{equation}
\begin{aligned}
    d_{\mathcal{M}_{0|k+1}}(x_{k+1}^1, x_{k+1}^2) &\leq \rho_{mpc} d_{\mathcal{M}_{0|k}}(x_k^1, x_k^2) + \gamma_z \| \delta z_{k+1} \|_{\infty} + w_x, \\
    \| \delta z_{k+1} \|_{\infty} &\leq \mathcal{A}_\delta \| \delta z_k \|_{\infty} + \gamma_x \max_{0 \leq \tau \leq L} d_{\mathcal{M}_{0|k-\tau}}(x_{k-\tau}^1, x_{k-\tau}^2) + \gamma_x^{n} \| \delta n \|_{\infty},
\end{aligned}
\label{eq:neural_bound}
\end{equation}
where $\gamma_z \triangleq \mathcal{S}_Y L_\Psi$, $\gamma_x \triangleq \frac{1}{\sqrt{\mu_2}} \mathcal{B}_\delta L_\Gamma$, $\gamma_x^{n} \triangleq \mathcal{B}_\delta L_\Gamma$, and $w_x \triangleq \sqrt{n_x \bar{m}}\|\delta d\|_\infty + \bar{d}_{shift}$ collects the bounded process and metric-shift perturbations.

\begin{theorem}[Interconnected Closed-Loop Practical Contraction]
\label{theorem:interconnected_dISS}
Consider the coupled dynamical system composed of the physical plant \eqref{eq:system} controlled by the parameterized differentiable MPC~\eqref{eq:problem} and the predictive Transformer network \eqref{eq:transformer-system}. Under the conditions of Lemma~\ref{lem:contraction_rate} ($\rho_{mpc} < 1$) and Theorem~\ref{dISS_stab} ($\mathcal{A}_\delta < 1$), the interconnected closed-loop system is practically exponentially contractive (cf. Lemma~\ref{lemma:contraction}) if
\begin{equation}
(1 - \rho_{mpc})(1 - \mathcal{A}_\delta) - \gamma_z \gamma_x > 0.
\label{eq:cond3}
\end{equation}
\end{theorem}

The formal proof is provided in the Appendix. Condition \eqref{eq:cond3} represents a coupled small-gain requirement: it ensures that the feedback interconnection between the Transformer parameter updates and the physical plant evolution does not amplify process disturbances and measurement noise. By structurally enforcing this condition as a regularization loss during training, Theorem~\ref{theorem:interconnected_dISS} guarantees the final robust closed-loop stability of the entire architecture.

%
%

\section{Experimental Validation}
\label{sec:exp}

To validate the theoretical guarantees established in Section~\ref{sec:small-gain}, We evaluate our method in three scenarios ---\emph{Quadrotor Target Reaching}, \emph{Quadrotor Obstacle Avoidance}, and \emph{Multi-Agent Navigation }--- spanning challenges from nonlinear dynamic stabilization to receding-horizon planning.
We leverage a differentiable MPC library~\citep{amatucci2025primal} in JAX on an NVIDIA RTX 5080.
In each scenario, we benchmark a standard MLP Actor-Critic MPC (AC-MPC), an unregularized Transformer Actor-Critic MPC (\emph{TAC-MPC}), an Open-Loop Stable TAC-MPC (\emph{OLS-TAC-MPC}) ablation (enforcing $\mathcal{A}_\delta < 1$ alone), and our proposed Closed-Loop Stable TAC-MPC (STAC-MPC), enforcing~\eqref{eq:cond3} via a soft-hinge structural regularization loss during PPO training. To avoid computing third-order dynamics derivatives ($\nabla^3 \Phi$) when backpropagating through~\eqref{eq:DRE}, the training regularizer evaluates a computationally efficient Gauss--Newton steady-state DARE surrogate ($\mathcal{H}_{i|k} \approx Q_{i|k}$). At evaluation, the closed-loop system is subjected to uniform observation noise $n_k$ and process disturbances $d_k$, injected as illustrated in Figure~\ref{fig:closed_loop_arch}. $h(\omega_{i|k}) \leq 0$ enforce in all the scenarios only the state and control boundary constraint.

Before detailing the empirical results, we first define the physical models and state spaces governing the two robotic systems utilized in these experiments:

\vspace{-0.4em}
\paragraph{\textbf{Quadrotor}} The discrete-time state vector at step $k$ is modeled as $x_k = [p_k^\top, v_k^\top, q_k^\top, \omega_k^\top]^\top \in \mathbb{R}^{13}$, encompassing position, linear velocity, orientation quaternion, and angular velocity. The control input $u_k = [f_k, \tau_k^\top]^\top \in \mathbb{R}^4$ consists of the total thrust scalar and the body torque vector. 

Applying a second-order discretization with sampling time $\Delta t$, the system dynamics are evaluated inline as $p_{k+1} = p_k + v_k \Delta t + \frac{\Delta t^2}{2} a_k$, $v_{k+1} = v_k + a_k \Delta t$, $q_{k+1} = q_k \otimes \exp([0, \frac{\Delta t}{2} \omega_k]^\top)$, and $\omega_{k+1} = \omega_k + \Delta t J^{-1} (\tau_k - \omega_k \times J \omega_k)$. The translational acceleration is defined as $a_k = -g e_3 + \frac{f_k}{m} R(q_k) e_3$, where $m$ is the mass, $J \in \mathbb{R}^{3 \times 3}$ is the inertia matrix, $g$ is gravity, $e_3 = [0, 0, 1]^\top$, $R(q_k)$ is the rotation matrix, and $\otimes$ denotes quaternion multiplication. 

Because the orientation evolves on the $SO(3)$ manifold, the state penalty within the MPC cost function is formulated using the attitude error with respect to the identity quaternion $q_I$ as $x_e = 2 \text{sign}(q_{ew}) q_{ev}$, where the relative error quaternion is $q_e = q_I^{-1} \otimes q $

We evaluate two quadrotor tasks over a $200$-step horizon with sampling time $\Delta t$ (full physical parameters and hyperparameters are detailed in the Appendix): (i) \emph{Target Reaching}, where the UAV starts from $p_0$ and must stabilize at $p^d$ in free space ; and (ii) \emph{Obstacle Avoidance}, where three cylindrical obstacles of radius $r_{\text{obs}} = 0.15\,\text{m}$ block the direct path between $p_0$ and $p^d$.

\vspace{-0.4em}
\paragraph{\textbf{Multi-Agent Navigation}} Each of the $N=4$ agents follows single-integrator dynamics $p_{k+1} = p_k + v_k \Delta t$ and is initialized in an antipodal cross-configuration at an initial distance of $3.0\,\text{m}$ from its goal (Fig.~\ref{fig:ma-results}a), tasked with swapping positions over $100$ steps while maintaining an inter-agent safety margin of $d_{\min} = 0.25\,\text{m}$. Each agent relies solely on local radial LiDAR range measurements and deploys a decentralized controller.
\subsection{Experimental Results}
Across the three scenarios, we evaluate $10$ quadrotor and $15$ multi-agent rollouts under severe perturbations to test whether our coupled small-gain condition~\eqref{eq:cond3} prevents closed-loop disturbance amplification. We inject uniform process disturbances $d_k \sim \mathcal{U}(-0.5, 0.5)^{12}$ and sensor noise $n_k \sim \mathcal{U}(-0.4, 0.4)^{12}$ on the quadrotor $SO(3)$ tangent space, and $d_k \sim \mathcal{U}(-0.3, 0.3)^{2 \times 4}$, $n_k \sim \mathcal{U}(-0.2, 0.2)^{2 \times 4}$ for the $4$-agent task. Figures~\ref{fig:quad-results} and~\ref{fig:ma-results} illustrate how theoretical contraction bounds translate to physical robustness:
\vspace{-0.4em}
\paragraph{Closed-loop regularization ensures robust convergence.} In \emph{Quadrotor Target Reaching} (Fig.~\ref{fig:quad-results}.a and ~\ref{fig:quad-results}.b), \emph{STAC-MPC} contracts from $1.75\,\text{m}$ to a steady-state error of $0.15 \pm 0.08\,\text{m}$, outperforming \emph{OLS-TAC-MPC} ($0.22 \pm 0.12\,\text{m}$) and unregularized \emph{TAC-MPC} ($0.85 \pm 0.15\,\text{m}$). Crucially, under distinct noise realizations, \emph{STAC-MPC} rollouts remain tightly confined within a $0.20\,\text{m}$ tube, empirically confirming convergence to the robust invariant set $\Omega$ (Lemma~\ref{lemma:contraction}). Without temporal filtering or Lipschitz bounds, memoryless \emph{AC-MPC} amplifies sensor noise through the MPC costs, diverging in $10/10$ trials ($1.3 \pm 1.7\,\text{m}$).
\vspace{-0.4em}
\paragraph{Theoretical robustness translates directly to safety.}
In \emph{Quadrotor Obstacle Avoidance} (Fig.~\ref{fig:quad-results}.c and ~\ref{fig:quad-results}.d), tight trajectory variance is critical because large state deviations cause physical collisions. By bounding the KKT parameter sensitivity $\mathcal{S}_Y$, \emph{STAC-MPC} prevents aggressive cost shifts under noise, navigating the clutter with zero collisions ($0.3 \pm 0.05\,\text{m}$ error). Conversely, sharp avoidance maneuvers in \emph{AC-MPC} and \emph{TAC-MPC} trigger severe KKT sensitivity spikes that amplify disturbances, causing $>10\times$ larger variance and divergence past step $75$.
\vspace{-0.4em}
\paragraph{Closed-loop stability is essential for multi-agent safety.}
In decentralized cross-navigation (Fig.~\ref{fig:ma-results}), local LiDAR partial observability exacerbates feedback coupling, causing $50.0\%$ and $37.5\%$ collision rates for \emph{AC-MPC} and \emph{TAC-MPC} (Fig.~\ref{fig:ma-results}.c). Crucially, while open-loop stable \emph{OLS-TAC-MPC} ($\mathcal{A}_\delta < 1$) cuts transit time by $12\%$ ($75$ vs.\ $85$ steps, Fig.~\ref{fig:ma-results}.b), ignoring the plant-network coupling gain $\gamma_z \gamma_x$ permits closed-loop instability, still yielding $33.3\%$ ($5/15$) collisions. Only \emph{STAC-MPC}, which enforces the coupled small-gain condition~\eqref{eq:cond3}, eliminates feedback amplification to achieve $0.0\%$ ($0/15$) collisions ($0.08 \pm 0.04\,\text{m}$ error).
\vspace{-0.4em}
\paragraph{Global contraction bounds trade transient agility for stability.}
In the quadrotor yaw error (Fig.~\ref{fig:quad-results}.a and ~\ref{fig:quad-results}.c, bottom), \emph{STAC-MPC} maintains a uniformly bounded error ($0.60 \pm 0.12\,\text{rad}$) across tasks. While unregularized \emph{TAC-MPC} achieves lower transient error ($0.25\,\text{rad}$) in free space (Fig.~\ref{fig:quad-results}.a), its unbounded parameter updates induce severe rotational oscillations ($>1.8\,\text{rad}$) near obstacles (Fig.~\ref{fig:quad-results}.c). Enforcing uniform contraction ($\rho_{\text{mpc}}$) and $\delta$ISS ($\mathcal{A}_\delta$) bounds inherently limits peak cost-variation rates—sacrificing minor transient agility to guarantee global closed-loop stability.

\begin{figure}[t]
\subfigure[]{\label{fig:tracking_hover}
    \includegraphics[width=0.23\linewidth]{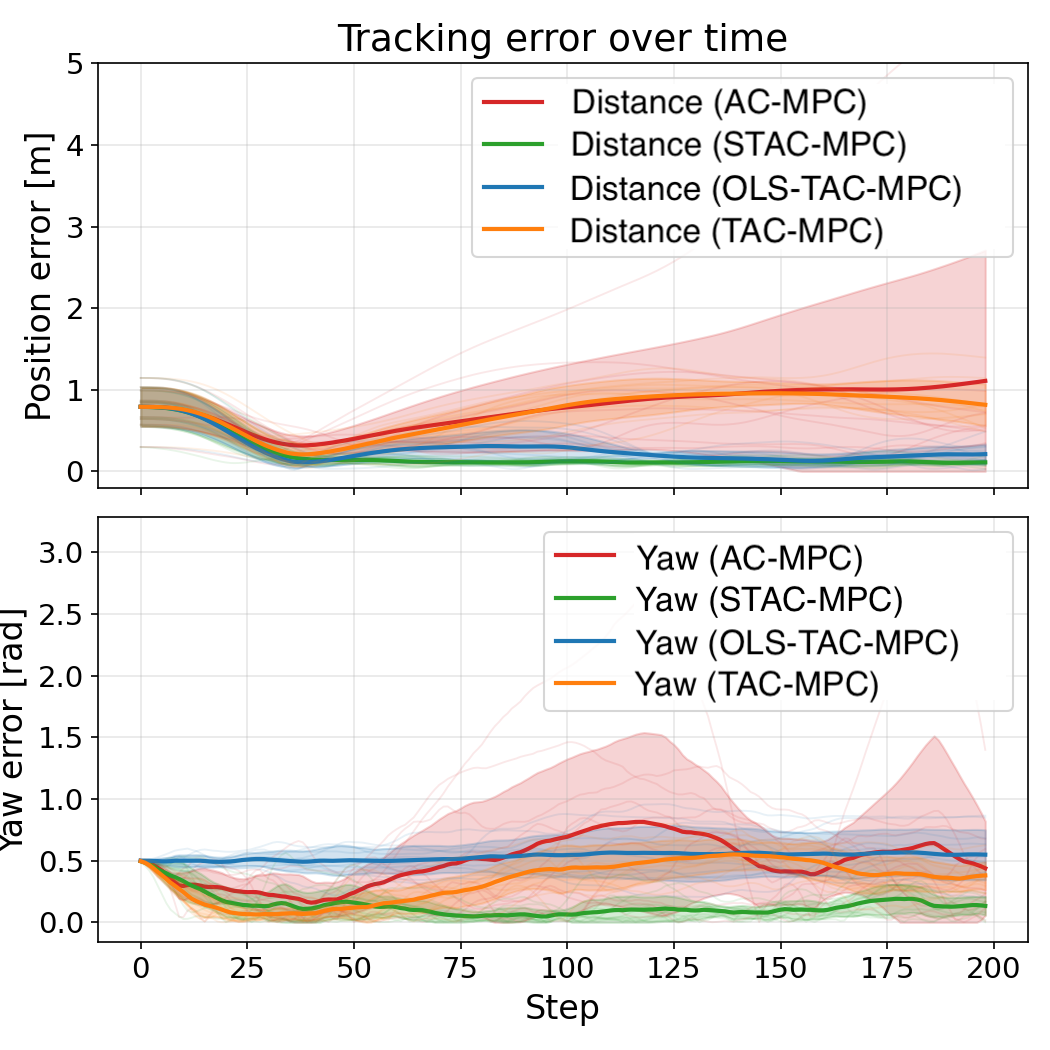}
    }
\subfigure[]{\label{fig:3d_hover}
    \includegraphics[width=0.23\linewidth]{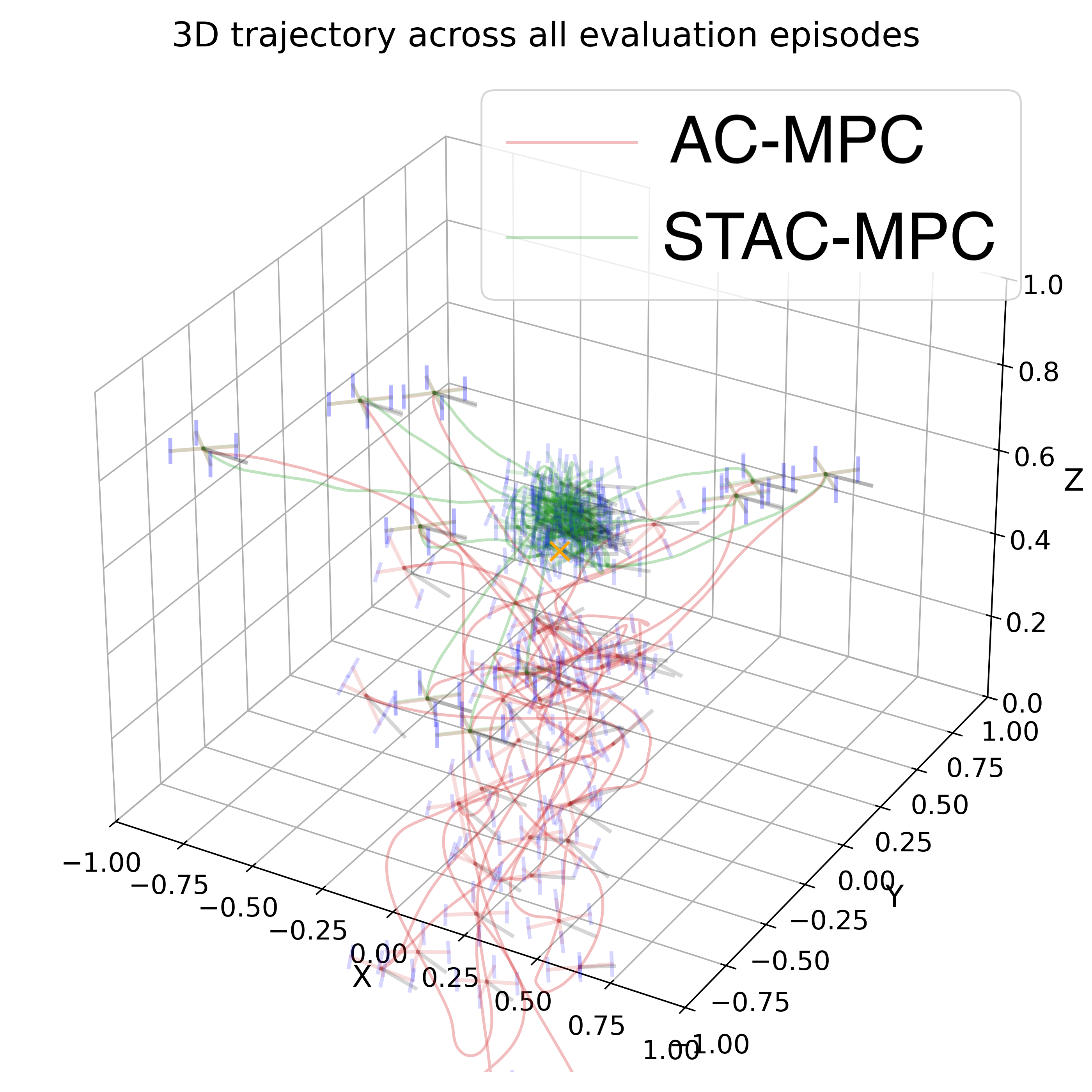}
    }
\subfigure[]{\label{fig:tracking_obstacle}
    \includegraphics[width=0.23\linewidth]{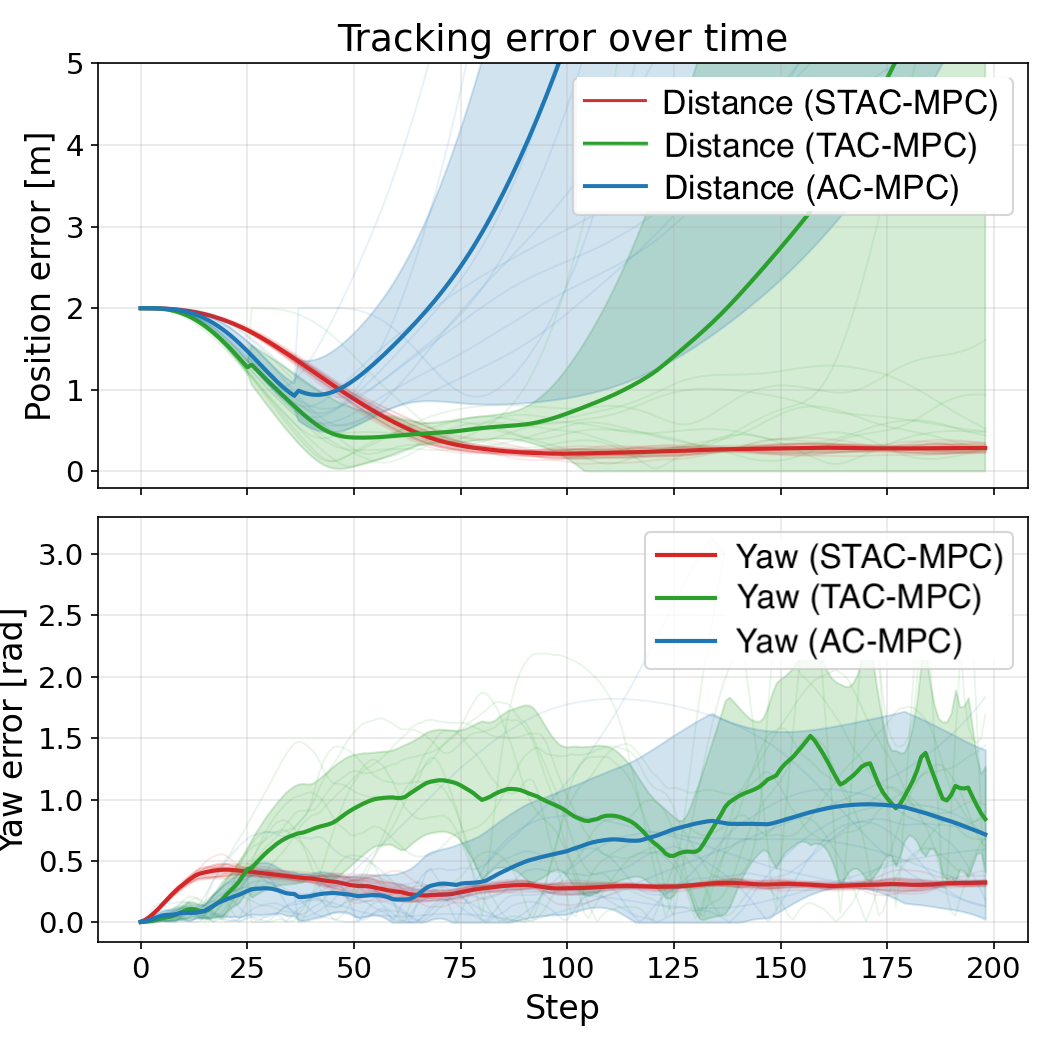}
    }
\subfigure[]{\label{fig:3d_obstacle}
    \includegraphics[width=0.23\linewidth]{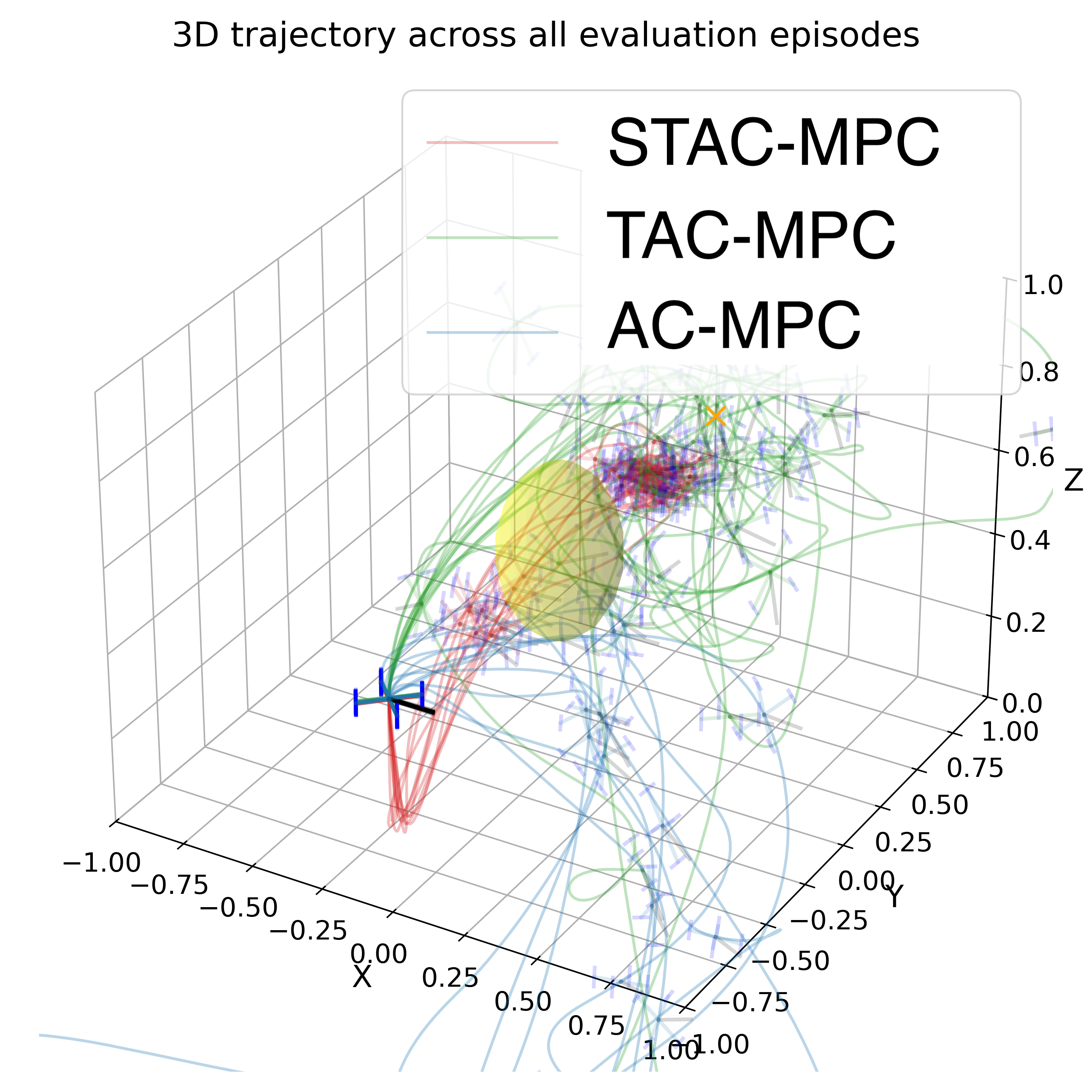}
    }
\caption{Target reaching 3D trajectories and tracking errors under uniform noise $n_k \sim \mathbb{U}(-0.4,0.4)^{12}$ and process disturbance $d_k \sim \mathbb{U}(-0.5,0.5)^{12}$. (a,b) \emph{Quadrotor Target Reaching}. (c,d) \emph{Quadrotor Obstacle Avoidance}.}
\label{fig:quad-results}
\end{figure}
\begin{figure}[t]
\centering
\subfigure[]{\label{fig:multi-agent_scenario}
\centering
    \includegraphics[width=0.23\linewidth, height=0.24\linewidth]{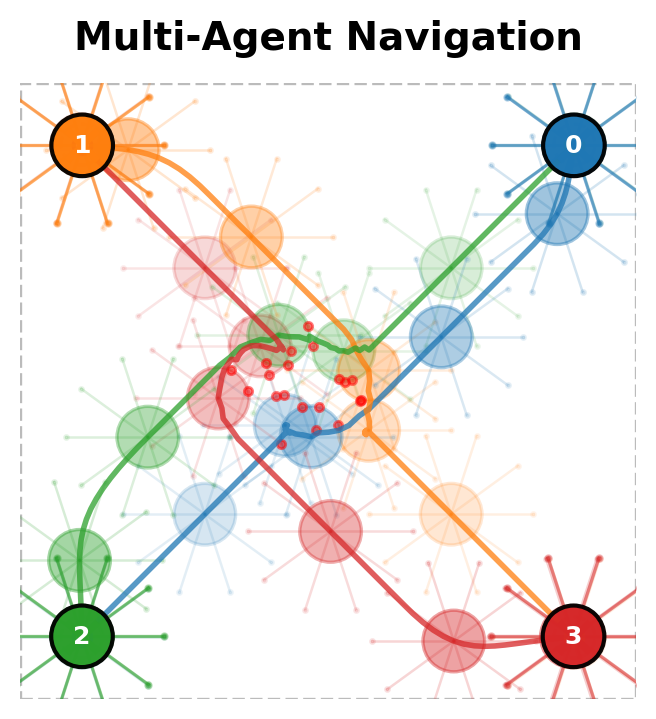}
    }\hfill
\subfigure[]{\label{fig:ma-error_trajectories}
    \centering
    \includegraphics[width=0.31\linewidth, height=0.25\linewidth]{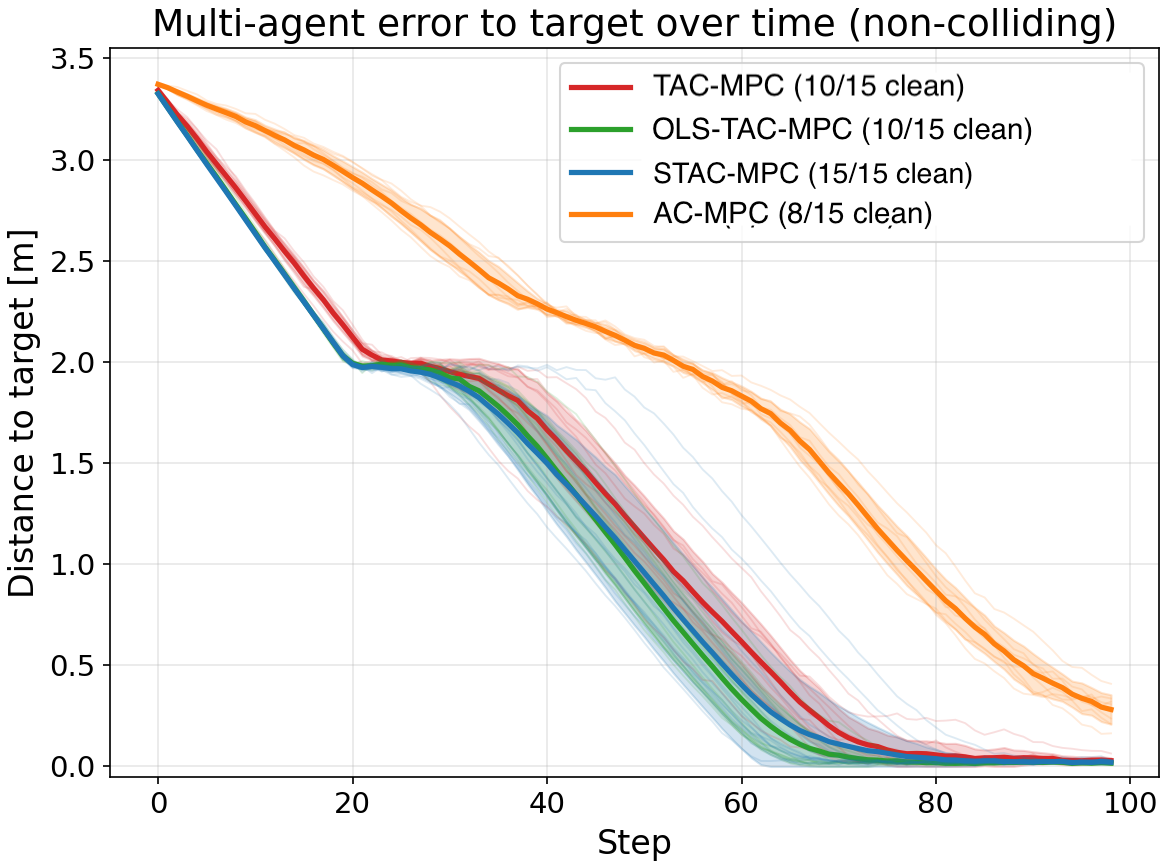}
    }\hfill 
\subfigure[]{\label{fig:ma_collision_rate}
\centering
    \includegraphics[width=0.30\linewidth, height=0.25\linewidth]{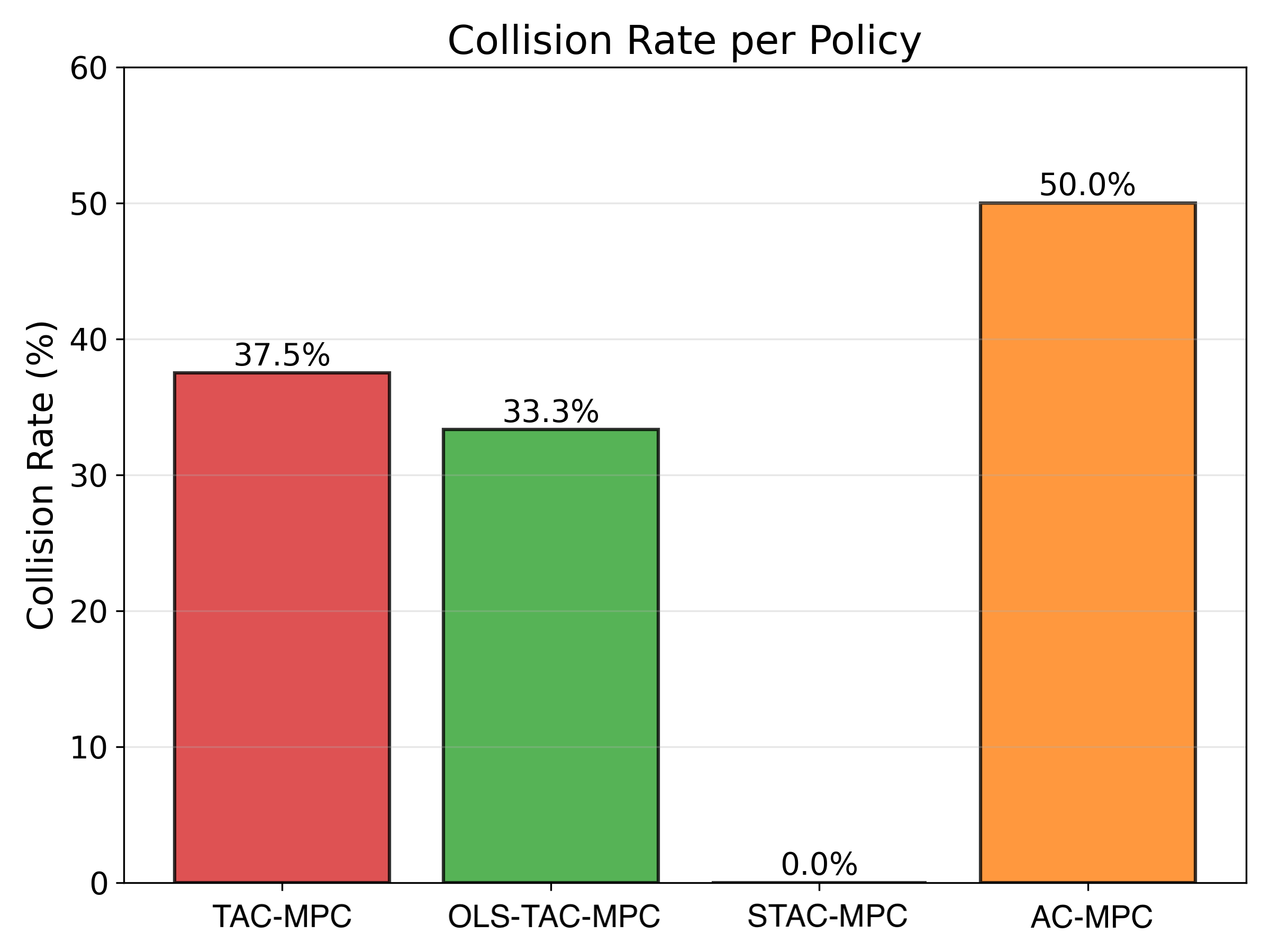}
    }
\caption{Multi-agent navigation task under $d_k \sim \mathbb{U}(-0.3,0.3)^{2 \times 4}$ and $n_k \sim \mathbb{U}(-0.2,0.2)^{2 \times 4}$ for $4$ agents: (a) cross-navigation scenario with LiDAR perception rays, (b) tracking error on collision-free trials, and (c) collision rate over $15$ evaluation trials.}
\label{fig:ma-results}
\end{figure}

\section{Conclusion}
\label{sec:concl}
This paper introduced a certifiably robust Transformer-Actor-Critic Model Predictive Control architecture. We established that sequence-modeling Transformers can be mathematically constrained to satisfy global incremental Input-to-State Stability ($\delta$ISS). By integrating these neural Lipschitz bounds with Riemannian contraction theory, we derived a coupled small-gain condition that guarantees uniform closed-loop contraction for the fully interconnected system. Numerical validation across highly nonlinear 3D quadrotor tasks and decentralized multi-agent navigation demonstrated that enforcing this theoretical bound during training yields a highly robust control policy. Our framework successfully rejected severe process disturbances and measurement noise to achieve a $0\%$ collision rate in obstacle-cluttered environments, whereas unregularized MLP and Transformer baselines catastrophically failed. Future work will explore state-dependent contraction metrics to mitigate trade-offs between strict global stability bounds and transient agility. Furthermore, we will extend this framework to output-feedback settings, leveraging partial contraction theory to account for unmeasurable variables.

\subsection*{Generative AI Use}
Generative AI tools were used for coding and improve paper writing. The authors
reviewed and verified all AI-assisted content and take responsibility for the
final content of the paper.

\bibliography{bibliography}

\iftrue

\section*{\Large Appendix}
\begin{table}[htpb]
    \centering
    \caption{Summary of Notation}
    \begin{tabular}{@{}ll@{}}
        \toprule
        \textbf{Symbol} & \textbf{Description} \\
        \midrule
        $\mathbb{R}, \mathbb{Z}_{\geq 0}$ & Real numbers and non-negative integers \\
        $(\cdot)_k$ & Variable at the discrete time step $k$ \\
        $(\cdot)_{i|k}$ & Prediction $i$ steps ahead from time $k$ \\
        $(\cdot)_{k-L:k}$ & Temporal sequence from step $k-L$ to $k$ \\
        $M \succ 0$ ($M \succeq 0$) & Positive (semi-)definiteness of matrix $M$ \\
        $\lambda_{\min}, \lambda_{\max}$ & Minimum and maximum eigenvalues \\
        $\|\cdot\|_2, \|\cdot\|_\infty$ & Euclidean and induced infinity norms \\
        $\|x\|_M^2$ & $M$-weighted squared norm ($x^\top M x$) \\
        $\|\cdot\|_{p \to q}$ & Mixed induced operator norm from $p$-norm to $q$-norm \\
        $\diag(v)$ & Diagonal matrix from vector $v$ \\
        $\diag(M)$ & Main diagonal of matrix $M$ as a vector \\
        $\partial / \partial x, \nabla_x(\cdot)$ & Partial derivative and Jacobian w.r.t. $x$ \\
        $\delta x$ & Difference between variables in a state space \\
        $\mathcal{KL}, \mathcal{K}_\infty$ & Continuous comparison function classes \\
        $\text{Concat}(\cdot)$ & Horizontal matrix concatenation \\
        $\text{softmax}(\cdot)$ & Row-wise normalized exponential \\
        $So(\cdot)$ & Smooth Softplus activation function \\
        \bottomrule
    \end{tabular}
    \label{tab:notation}
\end{table}

\subsection{Proof of Theorem~\ref{dISS_stab}}
\begin{proof}
Given two sequence trajectories $z_k^1, z_k^2 \in \mathbb{R}^{L \times D}$ with deviation $\delta z_k \triangleq z_k^1 - z_k^2$, and corresponding inputs $p_k^1, p_k^2 \in \mathbb{R}^{L \times D}$ with deviation $\delta p_k \triangleq p_k^1 - p_k^2$, the next-step deviation $\delta z_{k+1} \triangleq z_{k+1}^1 - z_{k+1}^2$ is expressed as:  
\begin{equation*}
    \begin{aligned}
    \delta z_{k+1} & = \alpha_2 g( \alpha_1 z_k^1+A(z_k^1,p_k^1)) + f(g(\alpha_1z_k^1+A(z_k^1,p_k^1))) \\
    &\quad - \alpha_2g(\alpha_1 z_k^2+A(z_k^2,p_k^2)) - f(g(\alpha_1 z_k^2+A(z_k^2,p_k^2))).
    \end{aligned}
\end{equation*}
Under the induced matrix infinity norm $\|z\|_\infty = \max_{1 \leq l \leq L} \|z_{l,:}\|_1$, each token row $x = z_{l,:} \in \mathbb{R}^D$ is measured in the vector $\ell_1$ norm. Token-wise LayerNorm $g(x) = \frac{\diag(c)}{\sqrt{\sigma^2 + \epsilon}} P x + \beta$~\citep{xu2019understanding, kim2021lipschitz} applies the centering projection $P = I - \frac{1}{D}\mathbf{1}\mathbf{1}^\top$, whose diagonal entries are $1 - \frac{1}{D}$ and off-diagonal entries are $-\frac{1}{D}$. Its exact induced operator norm is $\|P\|_{1 \to 1} = \|P\|_{\infty \to \infty} = (1 - \frac{1}{D}) + (D - 1)\frac{1}{D} = 2(1 - \frac{1}{D})$. Under the variance lower bound $\sigma \geq \sigma_{\min}$, the resulting LayerNorm Lipschitz constant is $L_g = \frac{2(1 - 1/D)\|c\|_{\infty}}{\sqrt{\sigma_{\min}^2 + \epsilon}}$. Combining $L_g$ with the submultiplicative MLP bound $L_f \leq \prod_{i=1}^{M}\Vert W_i \Vert_{\infty}$ and letting $\delta A_k \triangleq A(z_k^1,p_k^1) - A(z_k^2,p_k^2)$ yields:
\begin{equation}
\begin{aligned}
\|\delta z_{k+1}\|_{\infty} & \leq \alpha_2 L_g \| \alpha_1 \delta z_k + \delta A_k \|_{\infty} + L_f L_g \| \alpha_1 \delta z_k + \delta A_k \|_{\infty} \\
&\leq \left( \alpha_2 + \prod_{i = 1}^{M}\| W_i \|_{\infty }\right) \bigg( \frac{2(1 - 1/D)\|c\|_{\infty}}{\sqrt{\sigma_{\min}^2+\epsilon}}(\alpha_1 \|\delta z_k\|_{\infty} + \|\delta A_k\|_{\infty}) \bigg).
\end{aligned}
\label{eq:state_diff}
\end{equation}
Next, we bound the attention deviation $\delta A_k$. By submultiplicativity of the induced infinity norm, and observing that the induced infinity norm of horizontally concatenated matrices is bounded by the sum of their individual norms, we have:
\begin{equation}
\|\delta A_k\|_{\infty} \leq \|W_O\|_{\infty} \sum_{h=1}^H \|\delta \mathcal{H}_{h,k}\|_{\infty},
\end{equation}
where $\delta \mathcal{H}_{h,k} \triangleq \mathcal{H}_h(z_k^1, p_k^1) - \mathcal{H}_h(z_k^2, p_k^2)$ is the deviation for head $h$. For a single head $h$, the induced infinity norms of the projected queries, keys, and values satisfy $\|z_k W_Q^{(h)}\|_{\infty} \leq \|z_k\|_{\infty}\|W_Q^{(h)}\|_{\infty}$, $\|(p_k W_K^{(h)})^\top\|_{\infty} \leq L\|p_k\|_{\infty}\|W_K^{(h)}\|_{\infty}$, and $\|p_k W_V^{(h)}\|_{\infty} \leq \|p_k\|_{\infty}\|W_V^{(h)}\|_{\infty}$. Letting $\mathcal{H}_{h,k} \triangleq P_k V_k$, with softmax matrix $P_k$ and values $V_k \triangleq p_k W_V^{(h)}$, submultiplicativity yields $\|\delta \mathcal{H}_{h,k}\|_{\infty} \leq \|\delta P_k\|_{\infty}\|V_k\|_{\infty} + \|\delta V_k\|_{\infty}$, since $\|P_k\|_{\infty}=1$. Because the row-wise softmax Jacobian $J_{\text{softmax}}(\pi) = \diag(\pi) - \pi\pi^\top$ is symmetric with $\|J_{\text{softmax}}(\pi)\|_{1 \to 1} = \max_j 2\pi_j(1-\pi_j) \leq \frac{1}{2}$, we have $\|\delta P_k\|_\infty \leq \frac{1}{2\sqrt{D_h}}\|\delta(z_k W_Q^{(h)}(p_k W_K^{(h)})^\top)\|_\infty$. Applying Assumption~\ref{assumption1} ($\|p_k\|_\infty \leq 1$) and grouping the $\delta z_k$ and $\delta p_k$ terms gives:
\begin{equation*}
    \begin{aligned}
    \|\delta \mathcal{H}_{h,k}\|_{\infty} &\leq \left( \frac{L}{2\sqrt{D_h}} \|W_V^{(h)}\|_{\infty} \|W_Q^{(h)}\|_{\infty} \|W_K^{(h)}\|_{\infty} \right) \|\delta z_k\|_{\infty} \\
    &\quad + \Big( \|W_V^{(h)}\|_{\infty} + \frac{L}{2\sqrt{D_h}} \|W_V^{(h)}\|_{\infty} \|W_Q^{(h)}\|_{\infty} \|W_K^{(h)}\|_{\infty} \|z_k\|_{\infty} \Big) \|\delta p_k\|_{\infty}.
    \end{aligned}
\end{equation*}
Substituting this bound back into~\eqref{eq:state_diff}, we obtain:
\begin{equation*}
\begin{aligned}
    \|\delta z_{k+1}\|_{\infty} &\leq \left( \alpha_2 +\prod_{i = 1}^{M}\| W_i \|_{\infty}\right) \\
    &\quad \times \bigg( \frac{2(1 - 1/D)\|c\|_{\infty}}{\sqrt{\sigma_{\min}^2+\epsilon}} \Big( (\alpha_1 + \mathcal{L}_{MHA, z})\|\delta z_k\|_{\infty} + \mathcal{L}_{MHA, p} \|\delta p_k\|_{\infty} \Big) \bigg),
    \end{aligned}
    \label{eq:state_diff_2}   
\end{equation*}
where $\mathcal{L}_{MHA, z}$ is defined in Theorem~\ref{dISS_stab} and
\[\mathcal{L}_{MHA, p} \triangleq \|W_O\|_{\infty} \sum_{h=1}^H \left( \|W_V^{(h)}\|_{\infty} + \frac{L}{2\sqrt{D_h}} \|W_V^{(h)}\|_{\infty} \|W_Q^{(h)}\|_{\infty} \|W_K^{(h)}\|_{\infty} \|z_k\|_{\infty} \right).\] 

By defining $\mathcal{B}_{\delta} \triangleq \big( \alpha_2+\prod_{i = 1}^{M}\| W_i \|_{\infty}\big) \big(\frac{2(1 - 1/D)\|c\|_{\infty}}{\sqrt{\sigma_{\min}^2+\epsilon}} \mathcal{L}_{MHA, p} \big)$, we extract the incremental state bound stated in Theorem~\ref{dISS_stab} (where $\delta z_0 \triangleq z_0^1 - z_0^2$ denotes the initial state deviation):
\begin{equation}
\begin{aligned}
    \|\delta z_k\|_{\infty} \leq \mathcal{A}^k_{\delta} \|\delta z_0\|_{\infty} + (1-\mathcal{A}_{\delta})^{-1} \mathcal{B}_{\delta} \|\delta p\|_{\infty}.
    \end{aligned}
    \label{eq:incremental_trajectories}
\end{equation}
Consequently, the maximum distance between the state trajectories admits an asymptotic bound given by $\gamma_{\delta} = (1-\mathcal{A}_{\delta})^{-1} \mathcal{B}_{\delta} \|\delta p\|_{\infty}$. This concludes the proof, demonstrating that the system is incrementally ISS according to Definition~\ref{dISS_def}.
\end{proof}

\subsection{Proof of Theorem~\ref{theorem:interconnected_dISS}}
\begin{proof}
Unrolling the physical subsystem inequality in~\eqref{eq:neural_bound} for $\rho_{mpc} < 1$ yields the exponential $\delta$ISS bound:
\begin{equation*}
    d_{\mathcal{M}_{0|k}}(x_k^1, x_k^2) \leq \rho_{mpc}^k d_{\mathcal{M}_{0|0}}(x_0^1, x_0^2) + \frac{\gamma_z}{1 - \rho_{mpc}} \sup_{1 \leq j \leq k} \|\delta z_j\|_\infty + \frac{w_x}{1 - \rho_{mpc}},
\end{equation*}
establishing an asymptotic gain $\gamma_{xz} \triangleq \frac{\gamma_z}{1 - \rho_{mpc}}$ from $\|\delta z\|_\infty$ to $d_\mathcal{M}$. Conversely, substituting $\|\delta p_k\|_\infty$ into~\eqref{eq:incremental_trajectories} under $\mathcal{A}_\delta < 1$ shows that the history-aware Transformer subsystem (with finite sliding-window memory $L$) is exponentially $\delta$ISS with asymptotic gain $\gamma_{zx} \triangleq \frac{\gamma_x}{1 - \mathcal{A}_\delta}$ from $d_\mathcal{M}$ to $\|\delta z\|_\infty$ and exogenous sensor-noise bound $\frac{\gamma_x^n}{1 - \mathcal{A}_\delta}\|\delta n\|_\infty$.

By the discrete-time nonlinear small-gain theorem~\citep{russo2012contraction}, the interconnected feedback loop is exponentially $\delta$ISS, converging exponentially to a bounded invariant neighborhood proportional to $\Delta d = [w_x, \gamma_x^n\|\delta n\|_\infty]^\top$, whenever the loop-gain product satisfies $\gamma_{xz}\gamma_{zx} < 1$, which for $\rho_{mpc} < 1$ and $\mathcal{A}_\delta < 1$ rearranges directly to~\eqref{eq:cond3}. Equivalently, $\rho(\mathcal{M}_{sg}) < 1$ for the non-negative comparison matrix $\mathcal{M}_{sg} = \begin{bmatrix} \rho_{mpc} & \gamma_z; \gamma_x & \mathcal{A}_\delta \end{bmatrix}$. Notably, tracking the physical distance via the Riemannian path integral consolidates the metric condition number exclusively within the cross-coupling product ($\gamma_z \gamma_x \propto \sqrt{\bar{m}/\mu_2} = \sqrt{\kappa(M)}$), preserving the nominal contraction rate $\rho_{mpc}$.
\end{proof}

\subsection{Proof of Lemma~\ref{lem:metric_shift}}
\begin{proof}
Let $x_{1|k} = \Phi(x_k, u_k)$ denote the one-step nominal state prediction, so the realized state satisfies $x_{k+1} = x_{1|k} + d_k(x_k)$ with $\|x_{k+1} - x_{1|k}\|_\infty \leq \bar{d}$. Between consecutive steps $k$ and $k+1$, the input shift to the Transformer combines this disturbance with the nominal receding-horizon trajectory progression $\bar{\delta}_{nom} \triangleq \sup_{k} (\|x^d_{k+1} - x^d_k\|_\infty + \|x_{1|k} - x_k\|_\infty)$ (which vanishes at steady-state equilibrium regulation). Under Theorem~\ref{dISS_stab} on the compact domain of Assumption~\ref{assumption2}, the resulting cost-parameter shift $\Delta Y_i \triangleq Y_{i|k+1} - Y_{i+1|k}$ satisfies:
\begin{equation}
    \sup_i \|\Delta Y_i\|_\infty \leq L_{\mathcal{T}} (\bar{d} + \bar{\delta}_{nom}).
\end{equation}

Let $\mathcal{Z}_{i|k} \triangleq (\mathcal{H}_{i|k}, A_{i|k}, B_{i|k})$ collect the stage-wise Hessian and Jacobian data along the optimal primal-dual trajectory $\eta^*(x_k, Y_k)$. By Assumption~\ref{ass:kkt_invertibility} ($\Phi \in C^3$ and uniform KKT invertibility $\|K^{-1}\| \leq C_K$ on compact $\mathcal{X} \times \mathcal{U}$), $\eta^*$ and hence $\mathcal{Z}_{i|k}$ are uniformly Lipschitz with respect to $(x_k, Y_k)$, giving $\sup_{0 \leq i \leq T-1} \|\mathcal{Z}_{i|k+1} - \mathcal{Z}_{i+1|k}\|_2 \leq L_{\mathcal{Z}}(\bar{d} + \bar{\delta}_{nom})$. Writing the Generalized DRE~\eqref{eq:DRE} as $M_{i|k} = \mathcal{R}(M_{i+1|k}, \mathcal{Z}_{i|k})$, its local Lipschitz constants $(L_M, L_Y)$ on the compact positive-definite metric set yield for $\Delta M_i \triangleq M_{i|k+1} - M_{i+1|k}$:
\begin{equation}
    \|\Delta M_i\|_2 \leq L_M \|\Delta M_{i+1}\|_2 + L_Y L_{\mathcal{Z}}(\bar{d} + \bar{\delta}_{nom}).
\end{equation}
Noting that at the terminal boundary $M_{T|k} = \mathcal{R}(M_{T|k}, \mathcal{Z}^{\text{DARE}}_{T|k})$ is the fixed point of the terminal DARE operator and $\|M_{T|k+1} - M_{T|k}\|_2 \leq L_{DARE} L_{\mathcal{T}}(\bar{d} + \bar{\delta}_{nom})$, unrolling this finite recursion from $i=0$ to $T-1$ yields:
\begin{equation}
    \sup_{x \in \mathcal{X},\, k \in \mathbb{Z}_{\geq 0}} \|\Delta M_{shift}\|_2 \leq \left( L_M^T L_{DARE} L_{\mathcal{T}} + L_Y L_{\mathcal{Z}} \sum_{j=0}^{T-1} L_M^j \right) (\bar{d} + \bar{\delta}_{nom}) \triangleq \Delta \bar{M} < \infty,
\end{equation}
where the finite horizon $T < \infty$ guarantees convergence for any finite $L_M$, and $\bar{\delta}_{nom} = 0$ at steady-state target equilibria ($\hat{\omega}_{T|k}$ with $\beta^* = 0$) recovers $\Delta \bar{M} \propto \bar{d}$.
\end{proof}

\subsection{Proof of Lemma~\ref{lem:contraction_rate}}
\begin{proof}
We first relate the exact MPC policy Jacobian $\nabla_{x_k} u^*_{0|k}$ to the Generalized DRE~\eqref{eq:DRE}. Differentiating the KKT stationarity conditions $\nabla_\eta \mathcal{L}(\omega^*, \beta^*; x_k) = 0$ with respect to $x_k$ at fixed $Y_k$ (noting that the linear cost vectors $P_{i|k}$ drop out) and setting $\tilde{\lambda}_{i|k} \triangleq \frac{1}{2}\delta \beta_{i|k}$ gives $\delta x_{0|k} = \delta x_k$, terminal conditions $\tilde{\lambda}_{T|k} = Q_{x,T|k}\delta x_{T|k}$ and $\delta u_{T|k} = 0$ (since $Q_{u,T|k} \succ 0$), and stage equations for $i = T-1, \dots, 0$:
\begin{equation}
    \begin{aligned}
        \mathcal{H}_{xx,i|k}\delta x_{i|k} + \mathcal{H}_{xu,i|k}\delta u_{i|k} - \tilde{\lambda}_{i|k} + A_{i|k}^\top \tilde{\lambda}_{i+1|k} &= 0, \\
        \mathcal{H}_{ux,i|k}\delta x_{i|k} + \mathcal{H}_{uu,i|k}\delta u_{i|k} + B_{i|k}^\top \tilde{\lambda}_{i+1|k} &= 0.
    \end{aligned}
\end{equation}
Backward Riccati block elimination of the stage-wise KKT system~\citep{amos2018differentiable, giftthaler2018family} yields $\tilde{\lambda}_{i|k} = \widetilde{M}_{i|k}\delta x_{i|k}$ and $\delta u_{i|k} = -\widetilde{K}_{i|k}\delta x_{i|k}$, where $(\widetilde{M}_{i|k}, \widetilde{K}_{i|k})$ obey the exact same Generalized DRE~\eqref{eq:DRE} initialized at $\widetilde{M}_{T|k} = Q_{x,T|k}$ instead of $M_{T|k}$. In particular, $\nabla_{x_k} u^*_{0|k} = -\widetilde{K}_{0|k}$ and $A_{cl,k} = A_{0|k} - B_{0|k}\widetilde{K}_{0|k}$.

Next, let $S_{0|k} \triangleq \mathcal{H}_{uu,0|k} + B_{0|k}^\top M_{1|k} B_{0|k} \succ 0$. Completing the square in~\eqref{eq:DRE} at $i=0$ around $K_{0|k} = S_{0|k}^{-1}(\mathcal{H}_{ux,0|k} + B_{0|k}^\top M_{1|k} A_{0|k})$ gives the exact algebraic identity for $A_{cl,k}$:
\begin{equation}
    A_{cl,k}^\top M_{1|k} A_{cl,k} = M_{0|k} - \begin{bmatrix} I \\ -\widetilde{K}_{0|k} \end{bmatrix}^\top \mathcal{H}_{0|k} \begin{bmatrix} I \\ -\widetilde{K}_{0|k} \end{bmatrix} + (\widetilde{K}_{0|k} - K_{0|k})^\top S_{0|k} (\widetilde{K}_{0|k} - K_{0|k}).
\end{equation}
By Assumption~\ref{ass:kkt_invertibility}, $\mathcal{H}_{0|k} \succeq \mu_2 I_{n_x+n_u}$ implies $\begin{bmatrix} I^\top & -\widetilde{K}_{0|k}^\top \end{bmatrix} \mathcal{H}_{0|k} \begin{bmatrix} I \\ -\widetilde{K}_{0|k} \end{bmatrix} \succeq \mu_2 (I + \widetilde{K}_{0|k}^\top \widetilde{K}_{0|k}) \succeq \mu_2 I$ (which inductively also confirms $M_{i|k} \succeq \mu_2 I$, i.e., $\underline{m} = \mu_2$). Moreover, because $(\widetilde{M}_{i|k}, \widetilde{K}_{i|k})$ and $(M_{i|k}, K_{i|k})$ share identical stage operators~\eqref{eq:DRE} over $i = 0, \dots, T-1$ and differ solely at the terminal boundary ($\widetilde{M}_{T|k} = Q_{x,T|k}$ vs. $M_{T|k}$), uniform stabilizability implies exponential decay of the boundary mismatch $\|(\widetilde{K}_{0|k} - K_{0|k})^\top S_{0|k} (\widetilde{K}_{0|k} - K_{0|k})\|_2 \leq \varepsilon_T$ with $T$. For $\varepsilon_T < \mu_2$, we obtain:
\begin{equation}
    A_{cl,k}^\top M_{1|k} A_{cl,k} \preceq M_{0|k} - (\mu_2 - \varepsilon_T) I \preceq \left(1 - \frac{\mu_2 - \varepsilon_T}{\bar{m}}\right) M_{0|k} \triangleq \rho_{mpc}^2 M_{0|k}.
    \label{eq:mpc-rate-supp}
\end{equation}

Finally, to account for the metric shift between steps, evaluating $d_{\mathcal{M}_{0|k+1}}(x_{k+1}^1, x_{k+1}^2)$ along the $M_{1|k}$-minimizing geodesic $\ell^*$ via Lemma~\ref{lem:metric_shift} yields:
\begin{equation}
    d_{\mathcal{M}_{0|k+1}}(x_{k+1}^1, x_{k+1}^2) \leq d_{\mathcal{M}_{1|k}}(x_{k+1}^1, x_{k+1}^2) + \sqrt{\Delta\bar{M}} \int_0^1 \left\| \frac{\partial \ell^*}{\partial s} \right\|_2 ds.
\end{equation}
Since $M_{1|k} \succeq \mu_2 I$ and $\|x\|_\infty \leq \tilde{X}$ (Assumption~\ref{assumption2}), the Euclidean length of $\ell^*$ is bounded by $D_{\mathcal{X}} \triangleq 2\sqrt{(\bar{m}/\mu_2) n_x}\,\tilde{X}$. Defining $\bar{d}_{shift} \triangleq \sqrt{\Delta\bar{M}} D_{\mathcal{X}}$ and applying~\eqref{eq:mpc-rate-supp} yields:
\begin{equation}
    d_{\mathcal{M}_{0|k+1}}(x_{k+1}^1, x_{k+1}^2) \leq \rho_{mpc} d_{\mathcal{M}_{0|k}}(x_{k}^1, x_{k}^2) + \bar{d}_{shift},
\end{equation}
establishing practical exponential contraction under Lemma~\ref{lemma:contraction}.
\end{proof}

\section*{Experimental Setup and Implementation Details}
\label{sec:experimental_setup}

\subsection{Hardware and Software Infrastructure}
The entire learning and control pipeline was implemented in Python using JAX to leverage hardware-accelerated just-in-time (JIT) compilation and automatic differentiation. Training was conducted on a compute node equipped with a dedicated GPU NVIDIA RTX 5080. The environment simulation, policy updates, and differentiable control layers were heavily vectorized, running $2048$ parallel environments simultaneously for the quadrotor tasks and $512$ parallel environments for the multi-agent navigation task.

\subsection{Model Architecture and Differentiable MPC}
Our architecture tightly couples a deep reinforcement learning policy with a differentiable Model Predictive Control (MPC) module. The policy networks are structured as follows:

\begin{itemize}
    \item \textbf{Feature Extractor \& Transformer:} The sequential state history of length $L + 1 = 11$ is processed by a single-layer Transformer encoder. This layer features an embedding dimension of $128$, $2$ attention heads, and a feed-forward network (FFN) hidden dimension of $128$.
    \item \textbf{Actor-Critic Networks:} The output of the Transformer serves as the input to the Actor and Critic heads. In the multi-agent setting, we adopt a Centralized Training with Decentralized Execution (CTDE) paradigm (MAPPO), where each agent's decentralized Actor processes its local observation history (relative target position, normalized arena position, and $10$-bin LiDAR measurements within a $0.40\text{ m}$ range) with shared network weights, while the centralized Critic evaluates the concatenated global state. Both the Actor and Critic are Multi-Layer Perceptrons (MLPs) utilizing $\tanh$ activation functions.
    \item \textbf{Differentiable MPC Layer:} The policy outputs task parameters for the MPC layer. The MPC solver computes trajectories over a horizon of $H = 10$ steps, utilizing $1$ main solver iteration and $20$ Preconditioned Conjugate Gradient (PCG) iterations. Gradients are backpropagated through the MPC optimization directly into the upstream neural networks.
\end{itemize}

\subsection{Reward Formulation}
\paragraph{Quadrotor Environment.}
As introduced in the main text, the vehicle dynamics follow a 6-DoF rigid-body model of the Mikrokopter platform. The nominal physical parameters are: mass $m = 1.4\,\text{kg}$, diagonal moment of inertia $J = \text{diag}(0.022, 0.022, 0.020)\,\text{kg}\cdot\text{m}^2$, gravitational acceleration $g = 9.81\,\text{m/s}^2$, and an actuation delay of $\tau_{\text{delay}} = 2\,\text{ms}$. The center of mass (CoM) is nominally at the geometric origin $p_{\text{com}} = [0, 0, 0]^\top\,\text{m}$. To evaluate model robustness under domain randomization, parameter uncertainties are sampled uniformly at each episode: $m \in [1.0, 1.8]\,\text{kg}$, $J_{ii} \in [0.01, 0.05]\,\text{kg}\cdot\text{m}^2$, CoM displacements $p_{\text{com}} \in [-0.02, 0.02]^3\,\text{m}$ (inducing off-diagonal cross-coupling in the spatial inertia matrix), and actuation delay $\tau_{\text{delay}} \in [2, 10]\,\text{ms}$.

The control input $u = [f, \tau^\top]^\top \in \mathbb{R}^4$ is box-constrained with total collective thrust $f \in [4.0, 25.0]\,\text{N}$ (nominal hover thrust $f_{\text{hover}} = mg \approx 13.73\,\text{N}$) and body torques $\tau_x, \tau_y \in [-3.0, 3.0]\,\text{N}\cdot\text{m}$ and $\tau_z \in [-1.0, 1.0]\,\text{N}\cdot\text{m}$. Control actions output by the policy are normalized in $[-1, 1]^4$ and affine-mapped to $[u_{\min}, u_{\max}]$. Simulation integrates the continuous dynamics with time step $\Delta t = 0.02\,\text{s}$ ($50\,\text{Hz}$) over a $200$-step horizon ($4.0\,\text{s}$ total). Flight is bounded within a workspace volume of $[-1.0, 1.0] \times [-1.0, 1.0] \times [0.0, 1.0]\,\text{m}$. In the \emph{Obstacle Avoidance} task, cylindrical obstacles of radius $r_{\text{obs}} = 0.15\,\text{m}$ block the path between the initial position $p_0 = [0.0, -1.0, 0.5]^\top\,\text{m}$ and the target $p^d = [0.0, 1.0, 0.5]^\top\,\text{m}$. The policy operates on a sliding history buffer of length $L + 1 = 11$ containing the $17$-dimensional per-step observation vector $o_t = [(p^d - p)^\top, v^\top, q^\top, \omega^\top, e_{\text{yaw}}, (p_{\text{obs}} - p)^\top]^\top$, where $e_{\text{yaw}} \in [-\pi, \pi]$ is the yaw tracking error and $p_{\text{obs}}$ denotes the nearest obstacle coordinate.

To encourage smooth, stable, and accurate flight towards the target, we utilize a composite reward function. The reward formulation heavily relies on exponentially mapped cost terms that gate behaviors based on proximity to the target, preventing reward collapse and limit-cycle oscillations. The total quadrotor reward $R_t^{\text{quad}}$ at time $t$ is computed as a weighted sum of individual scoring components alongside a sparse completion bonus:
\begin{equation}
\label{eq:reward}
\begin{aligned}
R_t^{\text{quad}} &= w_1 S_{\text{pos}} + w_2 S_{\text{pos}} R_{\text{yaw}} + w_3 R_{\text{yaw}} S_{\text{pos}} S_{\omega} \\
    &\quad + w_4 S_{\text{pos}} S_{\text{hover}} + w_5 S_{\text{pos}} S_{\text{smooth}} \\
    &\quad + w_6 S_{\text{target\_stab}} + w_7 S_{\text{target\_settle}} + B_{\text{close}}
\end{aligned}
\end{equation}
The gating term for target proximity is defined as $G_{\text{near}} = \exp(-8.0 \| p - p^d \|)$. The velocity penalty $S_{\text{vel}} = \exp(-2.5 \| \mathbf{W}_v v \|)$ applies extra damping around the goal using a diagonal weight matrix $\mathbf{W}_v = \text{diag}(1, 1, 1.8)$ to penalize vertical drift more heavily.

\paragraph{Multi-Agent Navigation Environment.}
In the multi-agent navigation task ($N = 4$ disc-shaped robots of radius $r_{\text{robot}} = 0.15\,\text{m}$ operating in a $3.0 \times 3.0\,\text{m}$ arena with maximum speed $v_{\max} = 0.5\,\text{m/s}$), agent velocities are driven by single-integrator kinematics $v_i = v_{\max} u_i$ with normalized control commands $u_i \in [-1.0, 1.0]^2$ integrated at $\Delta t = 0.10\,\text{s}$ ($10\,\text{Hz}$). Positions are clamped within the arena boundaries $[r_{\text{robot}}, 3.0 - r_{\text{robot}}]^2$. Each robot is equipped with a $360^\circ$ planar LiDAR sensor with $K = 10$ angular bins and a maximum range of $d_{\text{lidar}} = 0.40\,\text{m}$, measuring boundary clearance via ray-disc intersections with other agents.

Under Centralized Training with Decentralized Execution (CTDE), each agent's policy accesses an observation history of length $L + 1 = 11$ steps composed of local observation vectors $o_{i, t} \in \mathbb{R}^{15}$: the relative goal offset $(p_i^d - p_i) \in \mathbb{R}^2$, scalar distance $\|p_i^d - p_i\| \in \mathbb{R}$, normalized arena position $\frac{p_i}{1.5} - 1.0 \in [-1, 1]^2$, and the $10$-dimensional LiDAR range reading $\in [0, 0.40]^{10}$. The centralized critic receives the global state vector $s_t \in \mathbb{R}^{61}$, consisting of the concatenated local observations of all $4$ agents ($4 \times 15 = 60$) and the normalized episode progress $t / T_{\max} \in [0, 1]$.

Agents receive a shared cooperative team reward $R_t^{\text{MA}}$ designed with strictly positive components to promote target reaching without inducing overly conservative stagnation:
\begin{equation}
\label{eq:reward_ma}
R_t^{\text{MA}} = w_{\text{prog}} S_{\text{close}} + w_{\text{goal}} B_{\text{goal}} + B_{\text{alive}}
\end{equation}
where $S_{\text{close}}$ sums the exponential closeness of each agent $i \in \{1, \dots, N\}$ to its assigned target $p_i^d$, $B_{\text{goal}}$ counts the number of agents within the goal acceptance radius $r_{\text{goal}} = 0.10\,\text{m}$ via the indicator function $\mathbb{I}(\cdot)$, and $B_{\text{alive}}$ provides a constant per-step survival bonus. An inter-agent safety margin is maintained at $d_{\min} = 0.25\,\text{m}$; any physical collision ($\| p_i - p_j \| < 2 r_{\text{robot}} = 0.30\,\text{m}$ for $i \neq j$) immediately terminates the episode with a discount factor of $0$, thereby strongly penalizing unsafe trajectories through the loss of future cumulative team rewards.
The exact definitions and weights for both environments are detailed in Table~\ref{tab:reward_terms}.
\begin{table*}[ht]
\centering
\caption{Reward Function Components and Gains for Quadrotor and Multi-Agent Environments}
\label{tab:reward_terms}
\begin{tabular}{@{}lllc@{}}

\toprule
\textbf{Component} & \textbf{Symbol} & \textbf{Formulation} & \textbf{Gain ($w_i$)} \\ \midrule
\multicolumn{4}{c}{\textit{Quadrotor Environment}} \\ \midrule
Position Score & $S_{\text{pos}}$ & $\exp(-2.0 \| p - p^d \|)$ & $3.0$ \\
Yaw Score & $R_{\text{yaw}}$ & $1.0 \text{ if } e_{\text{yaw}} < 0.05 \text{ else } \exp(-2.0 \, e_{\text{yaw}})$ & $1.0$ \\
Angular Vel Score & $S_{\omega}$ & $\exp(-1.0 \| \omega \|)$ & $0.6$ \\
Hover Score & $S_{\text{hover}}$ & $\exp(-1.5 \| u_{\text{norm}} - u_{\text{hover,norm}} \|)$ & $0.4$ \\
Smoothness Score & $S_{\text{smooth}}$ & $\exp(-1.5 \| u_{\text{norm}} - u_{\text{last,norm}} \|)$ & $0.2$ \\
Target Stability & $S_{\text{target\_stab}}$ & $G_{\text{near}} \cdot S_{\text{vel}} \cdot S_{\omega}$ & $1.2$ \\
Target Settle & $S_{\text{target\_settle}}$& $G_{\text{near}} \cdot S_{\text{smooth}}$ & $0.8$ \\
Close Bonus & $B_{\text{close}}$ & $4.0 \text{ if } \| p - p^d \| < 0.02 \text{ else } 0.0$ & N/A \\ \midrule
\multicolumn{4}{c}{\textit{Multi-Agent Navigation Environment ($N = 4$, $r_{\text{goal}} = 0.10\,\text{m}$, $r_{\text{robot}} = 0.15\,\text{m}$)}} \\ \midrule
Closeness Score & $S_{\text{close}}$ & $\sum_{i=1}^{N} \exp(-\| p_i - p_i^d \|)$ & $5.0$ \\
Goal Bonus & $B_{\text{goal}}$ & $\sum_{i=1}^{N} \mathbb{I}(\| p_i - p_i^d \| < r_{\text{goal}})$ & $10.0$ \\
Alive Bonus & $B_{\text{alive}}$ & $0.1 \text{ if } \min_{i \neq j} \| p_i - p_j \| \geq 2 r_{\text{robot}} \text{ else terminate}$ & N/A \\
\bottomrule
\end{tabular}
\end{table*}
\subsection{Training Hyperparameters}
The models were trained using Proximal Policy Optimization (PPO) for the single-agent quadrotor tasks and Multi-Agent PPO (MAPPO) under CTDE for the multi-agent navigation task. Environment returns were normalized during training. A comprehensive list of the RL, network architecture, MPC, and environment hyperparameters is provided in Table~\ref{tab:hyperparameters}.
\begin{table*}[ht]
\centering
\caption{Training, Architecture, and Environment Hyperparameters (Quadrotor / Multi-Agent)}
\label{tab:hyperparameters}
\renewcommand{\arraystretch}{1.1}
\begin{tabular}{@{}ll | ll@{}}
\toprule
\textbf{Hyperparameter} & \textbf{Quad / MA} & \textbf{Hyperparameter} & \textbf{Quad / MA} \\ \midrule
\multicolumn{2}{c|}{\textit{RL Algorithm (PPO / MAPPO)}} & \multicolumn{2}{c}{\textit{Network, MPC \& Environment}} \\ \midrule
Optimizer & Adam & Actor Layer Sizes & $[128]^3$ / $[256]^2$ \\
Learning Rate & $10^{-3}$ ($5\text{e-}4$) / $4\text{e-}4$ & Critic Layer Sizes & $[128]^3$ / $[256]^2$ \\
LR Annealing & True (Linear) & Activation Function & $\tanh$ \\
Number of Envs & $2048$ / $512$ & Transformer Layers / Heads & $1$ / $2$ \\
Episode Length & $200$ / $100$ steps & Transformer Hidden / FFN Dim & $128$ / $128$ \\
Total Timesteps & $1.2 \times 10^{8}$ / $1.0 \times 10^{8}$ & Sequence Length ($L + 1$) & $11$ \\
Update Epochs & $4$ & Stability Reg. Coef & $1 \times 10^{-3}$ \\
Minibatches & $32$ ($25$) / $40$ ($25$) & Small-Gain Margin Thresh. & $0.1$ / $0.01$ \\
Discount ($\gamma$) & $0.98$ ($0.99$) / $0.99$ & MPC Horizon ($H$) & $10$ \\
GAE ($\lambda$) & $0.95$ & MPC / PCG Iterations & $1$ / $20$ \\
Clip Epsilon ($\epsilon$) & $0.2$ ($0.1$) / $0.2$ ($0.1$) & Discretization Step ($\Delta t$) & $0.02\text{ s}$ / $0.10\text{ s}$ \\
Entropy Coef & $0.001$ & MA Arena \& Max Speed & $3 \times 3\text{ m}$, $0.5\text{ m/s}$ \\
Value Loss Coef & $0.5$ & MA Agents ($N$) \& Radius & $4$, $r_{\text{robot}} = 0.15\text{ m}$ \\
Max Grad Norm & $0.5$ & MA LiDAR Bins \& Range & $10\text{ bins}$, $0.40\text{ m}$ \\
Target KL & $0.13$ & Reward Normalization & True \\ \bottomrule
\end{tabular}
\end{table*}
\fi
\end{document}